\documentclass{article}

\usepackage{microtype}
\usepackage{graphicx}
\usepackage{subcaption}
\usepackage{booktabs} 
\usepackage{wrapfig}

\usepackage[preprint]{icml2026}
\usepackage[
  colorlinks=true,
  linkcolor=cyan,
  citecolor=cyan
]{hyperref}

\usepackage{url}

\usepackage{algorithm}
\usepackage{algorithmic}
\newcommand{\TO}{\textbf{to}}

\usepackage{amsmath}
\usepackage{amsfonts}
\usepackage{dsfont}
\usepackage{cancel}

\usepackage[most]{tcolorbox}
\tcbuselibrary{listingsutf8} 
\usepackage{booktabs}
\usepackage{multirow}
\usepackage{graphicx}
\usepackage{url}

\usepackage[table]{xcolor}
\usepackage[normalem]{ulem}

\newcommand{\first}[1]{\colorbox{green!25}{#1}}
\newcommand{\second}[1]{\colorbox{blue!15}{#1}}
\newcommand{\third}[1]{\colorbox{gray!20}{#1}}

\newcommand{\worst}[1]{\textcolor{red}{\uline{\textcolor{black}{#1}}}}

\usepackage{amsmath}
\usepackage{amssymb}
\usepackage{mathtools}
\usepackage{amsthm}

\usepackage{xcolor}
\usepackage{mdframed}

\definecolor{obj}{RGB}{220,245,255}    
\definecolor{cons}{RGB}{235,220,255}   

\newmdtheoremenv[
  backgroundcolor=white!95!teal!5,
  linecolor=blue!40!teal,
  linewidth=1pt,
  roundcorner=6pt,
  innertopmargin=5pt,
  innerbottommargin=5pt,
  innerrightmargin=10pt,
  innerleftmargin=10pt,
  frametitlebackgroundcolor=blue!10,
  frametitlefont=\bfseries,
  shadow=true,
  shadowsize=3pt,
  shadowcolor=black!15,
  skipabove=10pt,
  skipbelow=10pt
]{claimwithframe}{Theorem}

\newmdtheoremenv[
  backgroundcolor=blue!95!teal!5,
  linecolor=blue!40!teal,
  linewidth=1pt,
  roundcorner=6pt,
  innertopmargin=5pt,
  innerbottommargin=5pt,
  innerrightmargin=10pt,
  innerleftmargin=10pt,
  frametitlebackgroundcolor=blue!10,
  frametitlefont=\bfseries,
  shadow=true,
  shadowsize=3pt,
  shadowcolor=black!15,
  skipabove=10pt,
  skipbelow=10pt
]{proofwithframe}{Proof}

\usepackage{pifont}
\newcommand{\cmark}{\ding{51}}  
\newcommand{\xmark}{\ding{55}}  

\newtheorem{theorem}{Theorem}[section]

\newtheorem{lemma}[theorem]{Lemma}

\usepackage[textsize=tiny]{todonotes}

\icmltitlerunning{Planning with Language and Generative Models}

\begin{document}

\twocolumn[
  \icmltitle{Planning with Language and Generative Models: Toward General Reward-Guided Wireless Network Design}



  \icmlsetsymbol{equal}{*}

\begin{icmlauthorlist}
  \icmlauthor{Chenyang Yuan}{yyy}
  \icmlauthor{Xiaoyuan Cheng}{comp}
\end{icmlauthorlist}

\icmlaffiliation{yyy}{University of Sheffield}
\icmlaffiliation{comp}{University College London}

\vspace{-0.3em}
\begin{center}
Equal Contributing Authors
\end{center}



  \icmlcorrespondingauthor{Chenyang Yuan}{\url{cyuan7@sheffield.ac.uk}}
  \icmlcorrespondingauthor{Xiaoyuan Cheng}{\url{ucesxc4@ucl.ac.uk}}


  \vskip 0.3in
]



\printAffiliationsAndNotice{}  

\begin{abstract}
Intelligent access point (AP) deployment remains challenging in next-generation wireless networks due to complex indoor geometries and signal propagation. We firstly benchmark general-purpose large language models (LLMs) as agentic optimizers for AP planning and find that, despite strong wireless domain knowledge, their dependence on external verifiers results in high computational costs and limited scalability. Motivated by these limitations, we study generative inference models guided by a unified reward function capturing core AP deployment objectives across diverse floorplans. We show that diffusion samplers consistently outperform alternative generative approaches. The diffusion process progressively improves sampling by smoothing and sharpening the reward landscape, rather than relying on iterative refinement, which is effective for non-convex and fragmented objectives.
Finally, we introduce a large-scale real-world dataset for indoor AP deployment, requiring over $50k$ CPU hours to train general reward functions, and evaluate in- and out-of-distribution generalization and robustness. Our results suggest that diffusion-based generative inference with a unified reward function provides a scalable and domain-agnostic foundation for indoor AP deployment planning \footnote{Code and dataset will be shared soon.}.

\end{abstract}

\section{Introduction}

Intelligent network planning has become a vital AI-driven capability for future wireless communication systems \cite{giordani2020toward,saad2019vision,roh2014millimeter}, enabling automated and optimized deployment of network infrastructure. The long-term goal is to automatically determine physically feasible and performance-optimal access point (AP) configurations in real indoor environments based on floorplans, AP specifications, and realistic signal propagation models \cite{bakirtzis2024ai,fang2024self}.

\textbf{However, indoor AP planning is inherently difficult.} Indoor spaces contain complex and highly irregular geometries, diverse materials, and dense obstacles that create strong spatial heterogeneity in radio frequency propagation \cite{adickes2002optimization, aragon2017indoor}. These factors couple spatial structure with physical constraints, yielding a continuous, non-convex, and strongly multi-modal search space \cite{tam1995propagation, zhang2021fundamental}. As a result, small changes in AP placement can lead to large, non-linear coverage differences, and feasible regions are often fragmented or counterintuitive. This makes AP deployment a reasoning-intensive planning problem that is difficult to solve reliably by simple heuristics or convex optimization alone \cite{qiu2024large, yuanllm, wang2025large}.

\begin{table*}[h]
\caption{Comparison of methods across key capabilities. Convex optimization methods are inexpensive but have limited reasoning ability and often fall into suboptimal solutions in the highly non-convex AP planning space. LLM reasoning provides strong high-level inference and interactivity but requires frequent interaction with a physical simulator (verifier) for verification, resulting in high inference cost. In contrast, weighted sampling and diffusion methods offer efficient, scalable exploration of continuous spaces and are verifier-free, though their effectiveness depends strongly on the quality of the reward signal guiding the search.}
\centering
\small
\definecolor{lightcyan}{RGB}{220,245,245}
\begin{tabular}{lccccc}
\toprule
\textbf{Methods} &
\textbf{Reasoning abilities} &
\textbf{Suboptimal} &
\textbf{Inference cost} &
\textbf{Scalability} &
\textbf{Verifier} \\
\midrule
Convex optimization    & low  & \cmark & low  & low  & \xmark \\
LLM reasoning          & high & \xmark & high & high & \cmark \\
Weighted sampling      & high & \xmark & low  & high & \xmark \\
\rowcolor{lightcyan}
Diffusion sampling     & high & \xmark & low  & high & \xmark \\
\bottomrule
\end{tabular}
\label{tab:method_comparison}
\end{table*}

\textbf{LLM Reasoning.} Large language models (LLMs) \cite{kojima2022large,wei2022emergent, achiam2023gpt} have demonstrated remarkable general-purpose reasoning abilities, with prompting techniques enabling them to perform structured multi-step inference across diverse tasks, including math \cite{guo2025deepseek, ahn2024large}, physics \cite{lai2025panda}, and engineering \cite{xia2025buildarena}. Chain-of-thought prompting \cite{wei2022chain} allows models to decompose complex problems into intermediate steps, while strategies such as self-consistency aggregate multiple reasoning trajectories to improve reliability \cite{wang2022self}. Despite these advances, recent evaluations show that LLMs remain weak in spatial reasoning, struggling with distances, directions, and multi-step spatial tasks \cite{wu2024mind}. This poses a particular challenge for AP planning, which relies on accurate interpretation of floorplan geometry, obstacle layouts, and spatially varying signal propagation. These limitations motivate the need to systematically benchmark LLMs on the spatial and physical reasoning skills essential for AP deployment.

\textbf{Inference via Generative Models.} In parallel to LLM reasoning, generative inference models offer an alternative approach to plan by directly searching over the continuous solution space.  Unlike LLMs, which rely on semantic priors and symbolic reasoning, generative planners operate by maintaining and refining an explicit distribution over candidate placements \cite{ho2020denoising, janner2022planning, dong2024diffuserlite}. To be effective, however, these models require a clear guidance signal that shapes the distribution toward high-quality deployment configurations, typically in the form of a reward or objective function \cite{piche2018probabilistic, ubukata2024diffusion, akhound2024iterated, pan2024model, cheng2025safe}. Two representative families of methods widely used in spatial reasoning and continuous optimization are diffusion models and weighted sampling. Diffusion models refine noisy samples step by step toward better solutions, guided by the reward. Weighted sampling instead draws many candidates and keeps those with higher rewards. Both provide simple and effective ways to search continuous, multi-modal planning spaces \cite{lioutas2023criticsequentialmontecarlo,naesseth2014sequential}. However, a key research gap remains: due to the complexity and diversity of indoor environments, it is difficult to design a generalized reward function that can reliably guide planning across different floorplans and propagation conditions. In practice, the quality of the reward (or guidance signal) becomes the bottleneck, limiting how well generative planners can perform.

\begin{figure*}[ht]
    \centering
    \includegraphics[width=0.9\linewidth]{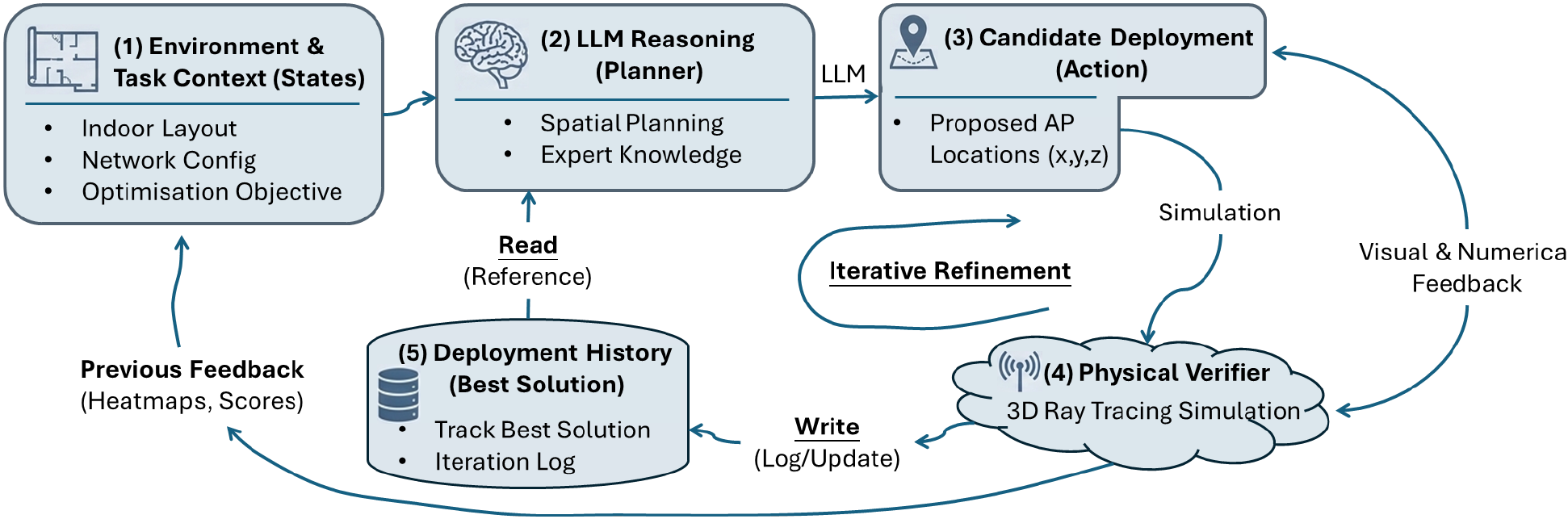}
    \caption{The LLM operates as a a agentic planner that takes the environment and task context as input, proposes candidate AP deployments, and iteratively refines them based on feedback from a physical verifier. Deployment actions are evaluated using 3D ray-tracing simulator (verifier), producing quantitative scores and visual coverage maps. Evaluation results are written to a deployment history module, which tracks the best solution and provides reference feedback for subsequent iterations, forming a closed-loop optimization.}
    \label{fig:agentic_flow}
\end{figure*}

Given this research gap, we revisit the AP planning problem from multiple perspectives (see summary in Table~\ref{tab:method_comparison}). We re-evaluate traditional optimization-based baselines, including both generic search strategies and gradient-based methods. The results indicate that the highly non-convex and fragmented structure of indoor deployment spaces makes classical optimization techniques alone ineffective when applied at scale. Building on these observations, our work makes three contributions. First, we benchmark LLMs on spatial reasoning tasks central to AP planning and demonstrate how agentic workflows can unlock the wireless domain knowledge embedded in these models. Second, we introduce a unified reward-guided evaluation framework to systematically compare diffusion-based and weighted-sampling generative planners. Our results show that reward design is the central bottleneck in generative inference for indoor AP planning. 
Crucially, diffusion sampling outperforms weighted sampling for reasons beyond iterative refinement.
Its advantage arises from the implicit mollification of the reward-induced score during the diffusion process, which yields a smoother optimization direction over highly fragmented and non-smooth reward landscapes.
Finally, we construct a $50k$ CPU-hour real-world dataset that enables the training of general reward functions and use it to demonstrate strong in-distribution and out-of-distribution generalization for indoor AP deployment. Our results highlight the potential of a single general reward function to serve as a scalable, domain-agnostic foundation for future indoor AP deployment planning.


\section{Preliminaries}
Before presenting our approach, we outline the problem formulation and the core elements that govern indoor AP planning. 

\textbf{Benchmark Problem Formulation.} We consider an indoor deployment planning problem over a bounded three-dimensional region $\mathcal{D} \subseteq \mathbb{R}^{3} $. A number $N \in \mathbb{N}$ of APs are deployed within this region. The deployment configuration is denoted by $\mathbf{P}=\left \{ \mathbf{p}_{1},\mathbf{p}_{2},\dots,\mathbf{p}_{N} \right \} $, where $\mathbf{p}_{n}=\left ( x_{n},y_{n},z_{n} \right ) $ represents the location of the $n$-th AP.

\textit{1. Coverage Modeling Based on Pathloss.} Indoor signal propagation from an AP experiences location-dependent attenuation caused by distance and obstacles, using the deterministic channel model. The coverage function $R_{n}$ defines a location as covered if its pathloss is below a threshold $\eta$, representing the maximum tolerable attenuation for reliable reception. 

\textit{2. Co-channel Interference.} In dense indoor environments, overlapping coverage of APs sharing on the same carrier frequency leads to co-channel interference. The aggregate interference at location $\mathbf{r} = \left ( x_{r},y_{r},z_{r}  \right )$ is denoted by $I\left ( \mathbf{r}  \right ) $, which is constrained to remain below a threshold $\xi$ to ensure robust shared-spectrum operation. 

\textit{3. Throughput Constraint.} While coverage and interference constraints ensure signal reachability and spatial deployment rationality, practical indoor services also require a minimum achievable data rate. The downlink throughput at location $\mathbf{r}$ is modeled as $T(\mathbf{r}) = B \log_2\!\left(1 + \gamma _{\mathbf{r}}\right)$, where $B$ denotes the system bandwidth and $\gamma _{\mathbf{r}}$ is the resulting signal-to-interference-plus-noise ratio (SINR). Service feasibility is enforced by requiring $ T\left ( \mathbf{r}  \right ) \ge T_{\mathrm{min} }$ over the area of interest. 

Combining the above conditions, the AP deployment problem is formulated as the following spatial optimization problem:
\begin{align}
\text{\textcolor{cyan}{(P1)}}\quad 
\max_{\mathbf{P}} \quad 
& \colorbox{obj}{$R(\mathbf{P})$} \label{eq:coverage} \tag{1a} \\[-0.2em]
\text{s.t.}\quad 
& \colorbox{cons}{$I(\mathbf{r}) \le \xi$}, 
&& \forall\, \mathbf{r} \in \mathcal{D} \label{eq:constraint_1} \tag{1b} \\[-0.2em]
& \colorbox{cons}{$T(\mathbf{r}) \ge T_{\min}$}, 
&& \forall\, \mathbf{r} \in \mathcal{D} \label{eq:constraint_2} \tag{1c} \\[-0.2em]
& x_n \in [x_{\min}, x_{\max}], 
&& \forall\, n \label{eq:constraint_3} \tag{1d} \\[-0.2em]
& y_n \in [y_{\min}, y_{\max}], 
&& \forall\, n \label{eq:constraint_4} \tag{1e} \\[-0.2em]
& z_n \in [z_{\min}, z_{\max}], 
&& \forall\, n \label{eq:constraint_5} \tag{1f} \\[-0.2em]
& \|\mathbf{p}_i - \mathbf{p}_j\| \ge d_{\min}, 
&& \forall\, i \neq j \in \mathbb{N} \label{eq:constraint_6} \tag{1g}
\end{align}

The optimization problem \textcolor{cyan}{(P1)} is highly \textbf{non-convex and non-continuous}. The non-convexity stems from the nonlinear pathloss model and the coupled interference and throughput constraints in \eqref{eq:constraint_1} and \eqref{eq:constraint_2}, while the non-continuity is primarily introduced by the minimum separation constraint in \eqref{eq:constraint_6} and the fragmented indoor domain $\mathcal{D}$. Moreover, the feasible region is further constrained by the three-dimensional deployment bounds in equations \ref{eq:constraint_3}\textcolor{cyan}{-}\ref{eq:constraint_5} and the pointwise service requirements. Consequently, obtaining a globally optimal solution using conventional convex optimization techniques is intractable.
Motivated by this observation, we investigate the use of LLMs and generative models for indoor AP deployment.

\section{Method}
Our method consists of two main components: (1) we first introduce a novel agentic framework for AP deployment with the assistance of a verifier; (2) we then propose a generalized reward function that enables generative inference without relying on any verifier. 

\subsection{Agentic Workflow for LLM Reasoning}

Figure~\ref{fig:agentic_flow} illustrates the proposed agentic workflow for indoor multi-AP deployment. The workflow is organized as a closed-loop reasoning system, where a large language model (LLM) interacts with a physical verifier through iterative proposal, evaluation, and refinement. Rather than directly optimizing numerical variables, the LLM operates as a high-level planner that reasons over structured environmental descriptions and feedback signals.
The workflow consists of five interconnected components:

\textbf{(1) Environment and Task Context (States).}  
The workflow starts with an environment statement encoding the indoor layout, network configuration, and optimization objective. This information defines the deployment domain, constraints, and performance criteria, forming the state that grounds the LLM’s reasoning. In later iterations, the state is augmented with feedback summaries from previous evaluations, such as coverage heatmaps and performance scores.
\textbf{(2) LLM Reasoning (Planner).}  
Conditioned on the current state, the LLM acts as a reasoning-driven planner that interprets spatial layouts, task objectives, and accumulated feedback. 
Guided by embedded expert knowledge, it reasons at the semantic and spatial level, i.e., prioritizing uncovered regions and avoiding redundant AP deployments, rather than performing fine-grained signal-level optimization.
\textbf{(3) Candidate Deployment (Action).}  
The LLM outputs a candidate deployment specified by AP locations $\mathbf{P}$ in a structured format. The proposed deployment respects the fixed AP count and deployment constraints, and serves as an actionable configuration for subsequent evaluation.
\textbf{(4) Physical Verifier (Evaluation).}  
The candidate deployment is evaluated by a physical verifier based on 3D ray-tracing simulation. The verifier produces quantitative performance metrics along with visual outputs, such as coverage heatmaps and binary coverage maps, which are used to assess deployment quality.
\textbf{(5) Deployment History (Memory).}  
Evaluation results are recorded in a deployment history module that logs iterations and tracks the best-performing deployment. This memory is accessible to the LLM in subsequent iterations, enabling reference to prior outcomes and supporting iterative refinement.

Through iterative interaction among these components, the workflow realizes an evaluation-driven refinement loop. Visual and numerical feedback from the verifier are fed back into the environment state, guiding the LLM to progressively improve AP placement decisions. The process terminates when a predefined stopping criterion is met, such as convergence of coverage performance or a maximum number of iterations, and the best recorded deployment is returned as the final solution. See examples in Appendix~\ref{append:demo_of_agentc_reasoning}.

\subsection{Inference via Generative Models}
The core of the generative model lies in the construction of an appropriate reward function. In our setting, we define the reward function $r(\mathbf{P}, \mathcal{D})$ for region $\mathcal{D}$ as a normalized scalar value given by
\begin{equation}
    r(\mathbf{P},\mathcal{D}) \propto \int_{\mathcal{D}} \colorbox{obj}{$R(\mathbf{P})$} \cdot \colorbox{cons}{$\mathds{1}_{\{ I(\mathbf{r}) \le \xi \}}$}  \cdot \colorbox{cons}{$\mathds{1}_{\{ T(\mathbf{r}) \ge T_{\min} \}}$} d \mathbf{r},
\end{equation}
where $R(\mathbf{P})$ denotes the coverage function in \eqref{eq:coverage}, $\mathds{1}_{\{ I(\mathbf{r}) \le \xi \}}$ and $\mathds{1}_{\{ T(\mathbf{r}) \ge T_{\min} \}}$ are indicator functions that enforce the interference and throughput constraints in \eqref{eq:constraint_1} and \eqref{eq:constraint_2}, respectively. The design and architecture of reward function are presented in Appendix~\ref{append:reward_design}.  The learned reward landscapes are illustrated in Figure~\ref{fig:reward_landscape}. It reflects the problem's nature is highly non-convex and non-continuous. 


\begin{figure}[h]
    \centering
    \includegraphics[width=0.95\linewidth]{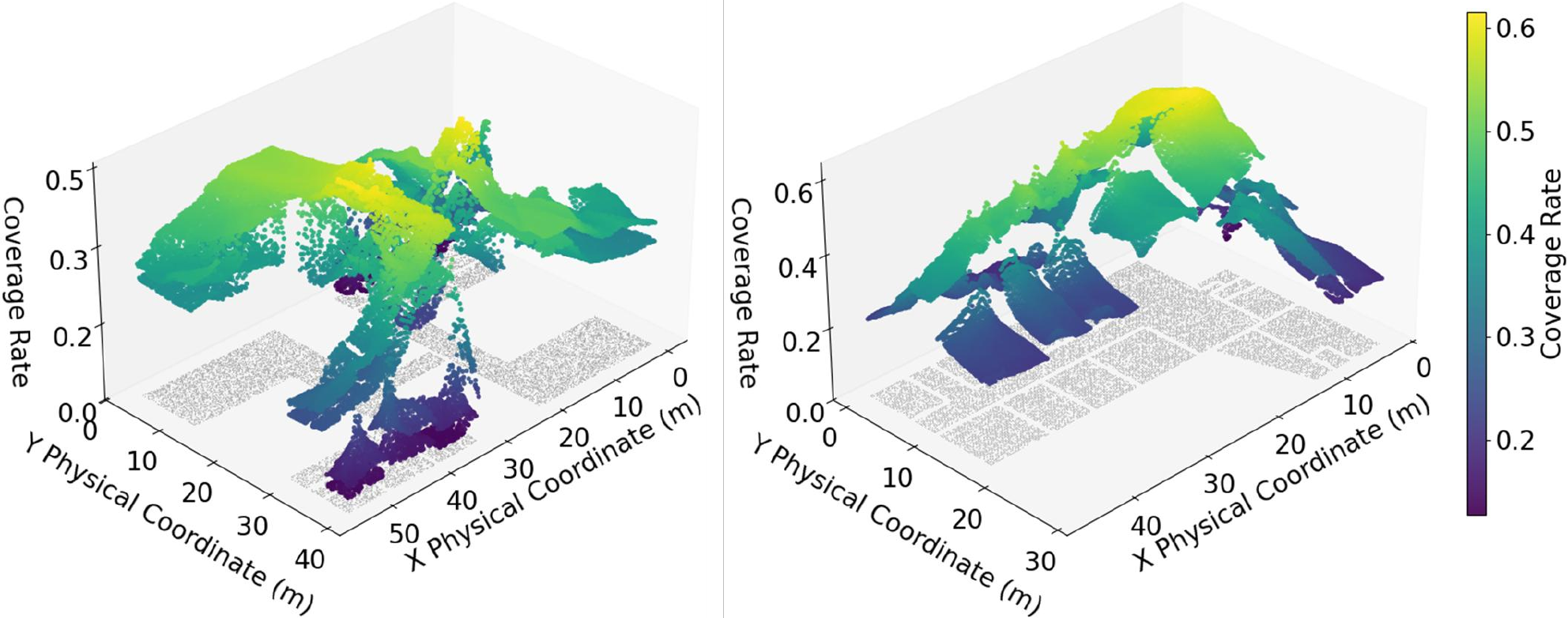}
    \caption{Non-convex and fragmented reward landscape over a given floor plan. The surface shows the reward value associated with each candidate access point (AP) location, where warmer colors indicate higher rewards. The shaded region projected onto the ground plane corresponds to the underlying floor plan, illustrating how spatial constraints shape the reward distribution.}
    \label{fig:reward_landscape}
\end{figure}
\vspace{-5pt}

Based on the above reward formulation, the distribution for region $\mathcal{D}$ over AP deployments is defined via Gibbs measure,
\begin{equation}
    p(\mathbf{P}|\mathcal{D}) = \frac{e^{\beta r(\mathbf{P}, \mathcal{D})}}{\mathcal{Z}}, \label{eq:gibbs_distribution}
\end{equation}
where $\beta$ is the temperature factor and $\mathcal{Z} = \int_{\mathbf{P}} e^{\beta r(\mathbf{P}, \mathcal{D}))} d\mathbf{P}$ is the partition function.  Sampling from the Gibbs distribution in \eqref{eq:gibbs_distribution} provides a principled way to generate high-quality AP deployment configurations. However, direct sampling is \textit{intractable} due to the partition function $\mathcal{Z}$. To address this challenge, we investigate two generative inference strategies: weighted and diffusion sampling.

\textbf{Weighted Sampling.} We first consider a weighted sampling strategy based on Sequential Monte Carlo (SMC) \citep{cappe2007overview}, which provides a principled mechanism for approximating the Gibbs distribution in \eqref{eq:gibbs_distribution}. In this framework, a population of particles (candidate AP deployments) is iteratively propagated and reweighted according to the reward function. 

Specifically, let $\{\mathbf{P}_t^{(k)}, w_t^{(k)}\}_{k=1}^K$ denote the set of particles and their associated weights at iteration $t$. At initialization ($t=0$), particles are sampled from a proposal distribution $q_0(\mathbf{P}|\mathcal{D})$, which is a standard Gaussian over the feasible deployment space, and are assigned equal weights $w_0^{(k)} = \frac{1}{K}$. At each subsequent iteration, particles are propagated according to a proposal transition kernel $q_t(\mathbf{P}_t | \mathbf{P}_{t-1}, \mathcal{D})$, yielding candidate deployments $\mathbf{P}_t^{(k)}$. The importance weight of each particle is then updated according to the Gibbs target distribution as
\begin{equation}
    \tilde{w}_t^{(k)} = w_{t-1}^{(k)} e^{\beta r(\mathbf{P}_t^{(k)}, \mathcal{D}))}, \; \text{normalize} \; 
w_t^{(k)} = \frac{\tilde{w}_t^{(k)}}{\sum_{j=1}^K \tilde{w}_t^{(j)}}. \label{eq:normalization}
\end{equation}
Adaptive resampling is performed based on the effective sample size. The expected AP deployment under the empirical particle approximation in \eqref{eq:normalization} is then given by
\begin{equation}
\mathbb{E}[\mathbf{P}_t|\mathcal{D}] = \sum_{k=1}^K w_t^{(k)} \mathbf{P}_t^{(k)} .
\end{equation} 
We choose different proposal distributions (i.e., Gaussian and Langevin) for sampling the target distribution (see Algorithms~\ref{alg:smc_gaussian} and \ref{alg:smc_langevin}).



    

\textbf{Diffusion Sampling.} Unlike weighted sampling, which relies on explicit reweighting of independently proposed samples, we adopt a diffusion-based generative approach that iteratively refines deployment configurations through a reverse variance-exploding (VE) stochastic differential equation (SDE) \citep{song2020score}. Rather than modeling an explicit forward noising process over AP placements, we directly operate on the reverse dynamics, where candidate deployments are progressively denoised and guided toward target distribution $p_0(\mathbf{P}_0|\mathcal{D})$.  In VE setting (see Appendix~\ref{append:diffusion}), the score function is represented as 

\begin{equation}\label{eq:score_function}
\begin{aligned}
\nabla_{\mathbf{P}} \log p_t (\mathbf{P}_t|\mathcal{D})
&= \frac{\mathbb{E}_{\mathbf{P}_{0|t} \sim \mathcal{N}(\mathbf{P}_t, \sigma_t^2I)}[\nabla_{\mathbf{P}} e^{\beta r(\mathbf{P}_{0|t}, \mathcal{D})}]}
        {\mathbb{E}_{\mathbf{P}_{0|t} \sim \mathcal{N}(\mathbf{P}_t, \sigma_t^2I)}[e^{\beta r(\mathbf{P}_{0|t}, \mathcal{D})}]}\, .
\end{aligned}
\end{equation}

Here, detailed derivation is shown in~\ref{eq:derivation_score_function}, $p_0(\mathbf{P}_0|\mathcal{D}) \propto e^{\beta r(\mathbf{P}_0,\mathcal{D})}$ denotes the target Gibbs distribution induced by the reward function, and $\mathbf{P}_{0|t}$ denotes a clean deployment sampled from the Gaussian posterior centered at the noisy configuration $\mathbf{P}_t$. The partition function $\mathcal{Z}$ cancels out in the score expression, yielding a normalized reward-guided score function that can be evaluated without requiring $\mathcal{Z}$.

Directly obtaining exact form in \eqref{eq:score_function} is intractable, and the score function can estimated via Monte Carlo estimation uses the same set of samples from $\mathcal{N}(\mathbf{P}_t, \sigma_t^2I)$, 
\begin{equation}\label{eq:score_function_estimation}
\begin{aligned}
\widehat{\nabla_{\mathbf{P}} \log p_t}(\mathbf{P}_t \mid \mathcal{D})
&\approx \frac{\sum_{k = 1}^K \nabla_{\mathbf{P}}e^{\beta r(\mathbf{P}_{0|t}^{(k)}, \mathcal{D})}}
{\sum_{k = 1}^K e^{\beta r(\mathbf{P}_{0|t}^{(k)}, \mathcal{D})}}.
\end{aligned}
\end{equation}

where $ \mathbf{P}_{0|t}^{(k)} \sim \mathcal{N}(\mathbf{P}_t, \sigma_t^2I),
\ \text{for all}\ k$, and $\mathbf{P}_{0|t}^k$ is an i.i.d particle from reparametrization trick (see \eqref{eq:score_function}) for estimating score function. Plugging the score function to reverse-time SDE in~\ref{eq:reverse_SDE} can sample AP configuration. See implementation in Algorithm~\ref{alg:reverse_diffusion}.

\textbf{Permutation Invariance in Sampling.}
Because APs are indistinguishable, permutations of a deployment yield equivalent configurations with identical rewards. Formally,
\[
r(\mathbf{P}, \mathcal{D}) = r(\pi(\mathbf{P}), \mathcal{D}), \quad \forall\, \pi \in \Pi,
\]
where $\Pi$ denotes the set of all permutation group actions over $\{1,\dots,N\}$ and $\pi(\mathbf{P}) = \{\mathbf{p}_{\pi(1)}, \dots, \mathbf{p}_{\pi(N)}\}$ represents the deployment obtained by reordering AP indices (see more details in Appendix~\ref{append:permutation_group}). This invariance property can be exploited to improve sample efficiency for both equations \ref{eq:normalization} and \ref{eq:score_function_estimation}. 

\begin{claimwithframe}[Reward is the Bottleneck.] \label{theorem:reward_bottleneck}
Let the target distribution over the region $\mathcal{D}$ be $p(\mathbf{P}, \mathcal{D}) \propto e^{\beta r(\mathbf{P}, \mathcal{D})}$ and its gradient satisfies the sub-Gaussian property. Then, with probability at least $1-2\delta$, the discrepancy between the generated AP deployment distribution $p$ and the target distribution $\hat{p}_0$ is bounded by
\begin{equation}
\begin{split}
        & D_{\text{KL}}( p(\mathbf{P}, \mathcal{D}) \| \hat{p}_0(\mathbf{P}, \mathcal{D})) \\
        \leq&  \underbrace{\mathcal{O}(\beta^4\varepsilon)}_{\text{error from reward estimation}} 
        + \underbrace{\mathcal{O}(T\log(1/\delta)/|\Pi|K)}_{\text{error from Monte Carlo estimation}}.
\end{split}
\end{equation}
Here, the first error source is drift part of reverse-time SDEs shaping by the true and learned reward functions (i.e., $- \frac{d \sigma_t^2}{dt} \nabla_{\mathbf{P}} \log p(\mathbf{P}_t)$), and the second error is from the Monte Carlo estimation in \eqref{eq:score_function_estimation}. $\varepsilon$ is the uniform bound with learned and true reward functions defined as $| r(\mathbf{P},\mathcal{D})- \hat{r}(\mathbf{P},\mathcal{D})| \leq \varepsilon$, and $|\Pi|$ is the cardinality of permutation group. 
\end{claimwithframe}

Theorem~\ref{theorem:reward_bottleneck} shows that the dominant source of distribution mismatch arises from reward estimation, rather than the diffusion sampler. The error bound of reward estimation is quadratically scaling as $\mathcal{O}(\beta^4 \varepsilon^2)$ with $| r(\mathbf{P}, \mathcal{D}) - \hat{r}(\mathbf{P}, \mathcal{D}) | \leq \varepsilon$ (see details in Theorem~\ref{thm:reward_score_gap}). While the Monte Carlo error decays as $\mathcal{O}(1/| \Pi |K)$ and can be reduced arbitrarily, the reward-induced error accumulates over time and cannot be mitigated by improved sampling alone. Consequently, the quality of the reward function fundamentally limits the performance of reward-guided generative planning. 

\begin{algorithm}[t]
\caption{Diffusion Sampling with VE SDE}
\label{alg:reverse_diffusion}
\begin{algorithmic}[1]
\REQUIRE Number of particles $K$, noise schedule $\{\sigma_t\}_{t=1}^T$, reward $r(\cdot,\mathcal{D})$
\STATE Initialize $\{\mathbf{P}_T^{(k)}\}_{k=1}^K \sim \mathcal{N}(\mathbf{0}, I)$
\FOR{$t = T$ \TO \ $1$}
    \FOR{$k = 1$ \TO \ $K$}
        \STATE Sample $\mathbf{P}_{0|t}^{(j)} \sim \mathcal{N}(\mathbf{P}_t^{(k)}, \sigma_t^2 I),\ j=1,\dots,K$
        \STATE Estimate permutation-invariant score
        \[
        \widehat{\nabla_{\mathbf{P}} \log p_t}
        =
        \frac{
        \sum_{j=1}^K \sum_{\pi \in \Pi}
        \nabla_{\mathbf{P}} e^{\beta r(\pi(\mathbf{P}_{0|t}^{(j)}),\mathcal{D})}
        }{
        \sum_{j=1}^K \sum_{\pi \in \Pi}
        e^{\beta r(\pi(\mathbf{P}_{0|t}^{(j)}),\mathcal{D})}
        }
        \]
        \STATE Reverse SDE update
        \[
        \mathbf{P}_{t-1}^{(k)} = \mathbf{P}_t^{(k)}
        + \frac{d\sigma_t^2}{dt} \widehat{\nabla_{\mathbf{P}} \log p_t}
        + \sqrt{\frac{d \sigma_t^2}{dt}}\,\boldsymbol{\epsilon},
        \quad \boldsymbol{\epsilon}\sim\mathcal{N}(\mathbf{0},I)
        \]
        \ENDFOR
    \ENDFOR
\STATE \textbf{return} $\{\mathbf{P}_0^{(k)}\}_{k=1}^K$
\end{algorithmic}
\end{algorithm}

\section{Experiment}
\label{experiment}

In this section, we conduct a comprehensive benchmark of agentic frameworks built on different large language models (LLMs) and evaluate the effectiveness of reward-guided generative inference. We further compare our approach against conventional algorithms and systematically examine the performance limits of the proposed methods.

\textbf{Benchmark Baselines.} We consider three categories of baselines for comparison.
(1) Agent using different LLMs. In this category, we instantiate the same agentic framework with different large language models, including DeepSeek-V3.2, GPT-4-Turbo, Claude-3-Haiku, Gemini-2.5-Flash, Grok-4, and Qwen-Plus, to evaluate the impact of the underlying LLM on planning performance.
(2) Gradient-based convex optimization. This baseline optimizes AP placement using gradient-based convex optimization on the learned reward function in our framework, serving as a numerical baseline.
(3) Generative methods. We additionally compare against generative approaches for AP placement. These include weighted sampling with different proposal distributions, as well as diffusion sampling. The dataset for reward training is shown in Appendix~\ref{append:data_generation}. 

\paragraph{Evaluation Metrics.}
To quantitatively evaluate AP deployment performance over the indoor region $\mathcal{D}$, we consider both interference robustness and service quality. Interference is measured using the \emph{Interference Occupancy Rate (IOR)}, defined as the average exceedance of co-channel interference above a threshold $\xi$:
\begin{equation}
\label{eq:ior}
\mathrm{IOR} \coloneqq \frac{1}{|\mathcal{D}|}\sum_{\mathbf{r}\in\mathcal{D}} [I(\mathbf{r})-\xi]_+ ,
\qquad
[x]_+ \coloneqq \max(0,x).
\end{equation}
A lower IOR indicates better compliance with interference constraints, with $\mathrm{IOR}=0$ corresponding to full compliance. Service capability is evaluated using the \emph{Throughput Quality Score (TQS)}, which jointly captures the typical throughput level and the spatial extent of service feasibility:
\begin{equation}
\label{eq:tqs}
\mathrm{TQS} \coloneqq
\operatorname{median}_{\mathbf{r}\in\mathcal{D}} T(\mathbf{r})
\cdot
\frac{1}{|\mathcal{D}|}\sum_{\mathbf{r}\in\mathcal{D}}
\mathbb{I}\!\big(T(\mathbf{r}) \ge T_{\min}\big).
\end{equation}
Together, IOR and TQS provide complementary scalar metrics for comparing deployment quality across different planning methods, reflecting both interference control and building-wide throughput performance.

\textbf{Tasks.} We evaluate all methods on indoor AP deployment tasks across four building levels with increasing complexity, covering civil, commercial, and industrial scenarios. Building Level 1 represents a simple environment with minimal structural complexity, Building Level 2 corresponds to a moderately complex environment with denser partitions and obstacles, and Building Level 3-4 models a highly complex industrial environment with severe obstructions and interference. For each building level, we consider multiple AP deployment scales and report runtime, coverage, IOR, TQS, and success rate to enable systematic comparison under increasing environmental complexity. 

\begin{figure*}[h]
    \centering
    \includegraphics[width=0.96\linewidth]{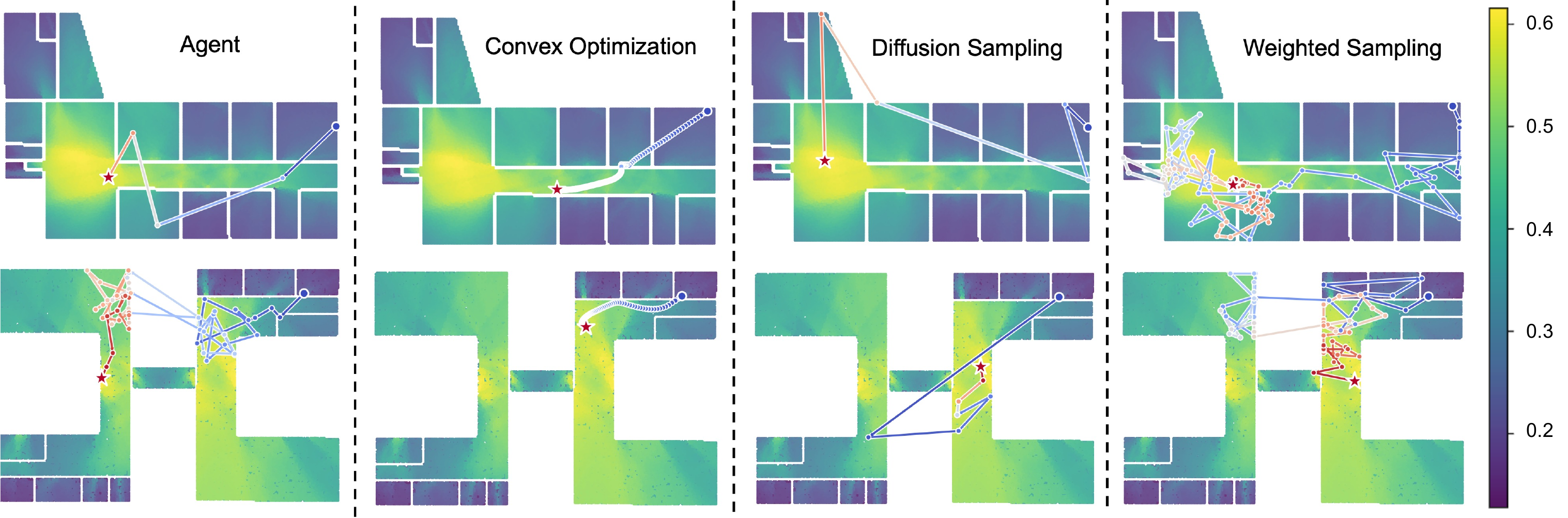}
    \caption{AP placement planning results on two indoor floorplans. The background colormap represents the reward heatmap, where warmer colors indicate higher expected deployment rewards. The star marker denotes the selected AP location, while the trajectories illustrate the intermediate solution refinement processes of different methods. From left to right, we compare the proposed agentic method with convex optimization, diffusion sampling, and weighted sampling. The results reveal distinct planning behaviors across methods and demonstrate the effectiveness (less iterative steps) of the diffusion sampling in identifying high-reward AP placements under complex indoor layouts.}
    \label{fig:AP_iternative_refinment}
\end{figure*}

\begin{figure*}[h]
    \centering
    \includegraphics[width=0.93\linewidth]{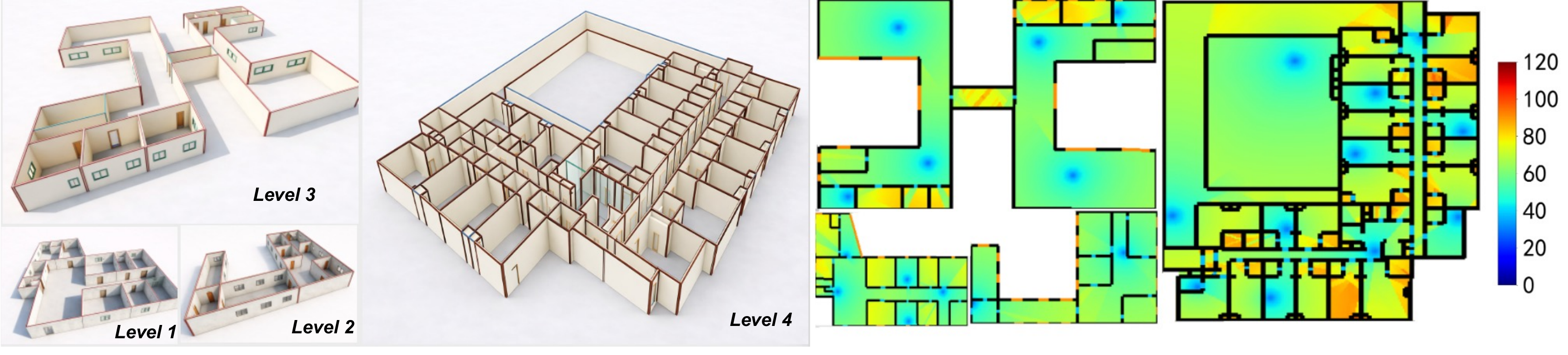}
    \caption{AP planning based on diffusion sampling over a multi-level building. The left panels show the 3D building geometry from Level 1 to Level 4. The right panel visualizes the final deployment produced by diffusion-based planning, where the heat map indicates the pointwise satisfication for building and blue nodes denote the AP locations. See visualization of intermediate reasoning and iterative optimization states in Figures~\ref{fig:reasoning_level_1}, \ref{fig:reasoning_level_2}, \ref{fig:reasoning_level_3}, \ref{fig:reasoning_level_4} in Appendix~\ref{append:Visualization of Intermediate Reasoning and Iterative Optimization States}.}
    \label{fig:diffusion_scaling}
\end{figure*}

\textbf{Result Analysis.} Our analysis is centered on addressing the key research gaps of this work. 

\emph{1. Are conventional optimization methods sufficient for solving the AP deployment problem?}
The answer is \textbf{NO.} The AP deployment problem \textcolor{cyan}{(P1)} induces a highly \emph{non-convex} and severely \emph{fragmented} optimization landscape as shown in Figure~\ref{fig:reward_landscape}, rendering conventional optimization methods inadequate in practice. From the perspective of convex optimization, this inadequacy is fundamental. Convex formulations rely on smooth objectives with a single global optimum and well-behaved gradients, assumptions that are systematically violated in AP deployment. Irregular floorplan geometries, obstacles, and complex propagation effects introduce discontinuities and isolated local optima in the objective landscape (see Figure~\ref{fig:AP_iternative_refinment}), causing convex relaxations or local descent methods to either converge to suboptimal solutions or fail to provide meaningful guarantees.

\begin{table*}[H]
\centering
\caption{Performance comparison under different AP deployment scales in Building Level~4.
Results are reported as mean$\pm$std. Top three methods are highlighted, and the worst is underlined.}
\label{tab:b14_apscale}

\setlength{\tabcolsep}{3pt}
\renewcommand{\arraystretch}{0.9}

\resizebox{\textwidth}{!}{
\begin{tabular}{ccccccc}
\toprule
\textbf{APs} & \textbf{Method} 
& \textbf{Runtime (s)$\downarrow$} 
& \textbf{Coverage (\%)$\uparrow$} 
& \textbf{IOR$\downarrow$} 
& \textbf{TQS$\uparrow$} 
& \textbf{Success (\%)$\uparrow$} \\
\midrule

\multirow{10}{*}{3}
& Diffusion       & \first{31.66$\pm$27.47} & \third{58.98$\pm$0.64} & \third{0.45$\pm$1.69} & 9.31$\pm$4.27 & \first{82.0} \\
& SMC             & \third{71.37$\pm$52.31} & 56.91$\pm$0.47 & 0.97$\pm$1.32 & 6.57$\pm$3.66 & \third{78.0} \\
& SMC-grad        & \second{62.01$\pm$70.45} & 57.15$\pm$0.47 & \second{0.26$\pm$0.37} & \first{10.97$\pm$3.35} & \second{80.7} \\
& DeepSeek-V3.2   & 195.43$\pm$15.33 & 57.35$\pm$0.20 & 1.15$\pm$1.14 & 8.23$\pm$3.65 & 74.0 \\
& GPT-4-turbo     & 597.20$\pm$619.10 & 56.92$\pm$1.50 & 0.96$\pm$1.18 & 9.91$\pm$4.41 & 72.7 \\
& Claude-3-Haiku  & 797.81$\pm$708.08 & 58.00$\pm$6.45 & 1.04$\pm$1.64 & \third{9.93$\pm$5.64} & 70.7 \\
& Gemini-2.5-flash& 225.15$\pm$201.89 & \second{61.69$\pm$10.56} & 1.45$\pm$2.14 & 7.45$\pm$5.83 & 75.3 \\
& Grok-4          & 284.17$\pm$220.43 & 57.51$\pm$0.37 & \first{0.01$\pm$0.23} & \second{10.87$\pm$0.06} & 76.0 \\
& Qwen-plus       & \worst{1142.55$\pm$903.67} & \first{61.89$\pm$12.33} & 1.23$\pm$2.13 & 6.57$\pm$3.66 & 69.3 \\
& Convex Optimization     & 205.54$\pm$26.13 & \worst{44.54$\pm$9.09} & \worst{1.54$\pm$2.17} & \worst{3.93$\pm$2.20} & \worst{58.0} \\
\midrule

\multirow{10}{*}{6}
& Diffusion       & \first{39.91$\pm$12.71} & 67.43$\pm$1.03 & 4.11$\pm$0.51 & \first{16.40$\pm$0.21} & \first{74.7} \\
& SMC             & 172.62$\pm$198.39 & \third{67.74$\pm$0.82} & \third{2.49$\pm$2.23} & 9.73$\pm$6.95 & \third{70.0} \\
& SMC-grad        & \second{128.02$\pm$124.46} & \second{67.86$\pm$0.87} & \second{1.43$\pm$0.50} & \third{13.16$\pm$0.59} & \second{72.7} \\
& DeepSeek-V3.2   & 193.73$\pm$94.54 & 67.62$\pm$1.45 & 5.11$\pm$0.91 & 2.87$\pm$0.83 & 65.3 \\
& GPT-4-turbo     & 337.70$\pm$168.25 & 66.56$\pm$0.96 & \first{1.31$\pm$0.51} & \second{13.42$\pm$5.75} & 64.0 \\
& Claude-3-Haiku  & 333.56$\pm$82.36 & 67.59$\pm$1.76 & 3.05$\pm$1.72 & 6.90$\pm$5.23 & 62.0 \\
& Gemini-2.5-flash& \third{140.21$\pm$123.65} & 66.49$\pm$1.92 & 4.46$\pm$0.58 & 3.02$\pm$0.81 & 66.7 \\
& Grok-4          & \worst{620.79$\pm$266.21} & 66.49$\pm$1.92 & 4.67$\pm$0.62 & 10.20$\pm$0.76 & 67.7 \\
& Qwen-plus       & 430.68$\pm$214.13 & \first{68.22$\pm$1.45} & \worst{5.32$\pm$0.54} & \worst{2.46$\pm$1.14} & 61.0 \\
& Convex Optimization     & 267.58$\pm$85.57 & \worst{60.67$\pm$7.99} & 5.24$\pm$3.08 & 3.16$\pm$0.59 & \worst{50.7} \\
\midrule

\multirow{10}{*}{8}
& Diffusion       & \first{102.26$\pm$71.57} & \second{74.52$\pm$1.24} & \first{2.73$\pm$0.23} & \first{13.25$\pm$2.75} & \first{68.0} \\
& SMC             & 346.75$\pm$352.49 & 73.31$\pm$0.70 & \second{3.80$\pm$0.63} & 3.89$\pm$2.13 & \third{63.3} \\
& SMC-grad        & 331.34$\pm$222.24 & 73.85$\pm$0.67 & 4.24$\pm$1.87 & \third{5.56$\pm$2.24} & \second{66.0} \\
& DeepSeek-V3.2   & 262.15$\pm$146.74 & 72.49$\pm$1.51 & \worst{6.54$\pm$0.23} & 3.68$\pm$1.23 & 58.7 \\
& GPT-4-turbo     & 243.29$\pm$158.90 & 72.41$\pm$1.67 & \third{3.84$\pm$1.80} & \second{12.95$\pm$8.77} & 57.3 \\
& Claude-3-Haiku  & \worst{689.94$\pm$331.97} & 72.90$\pm$3.17 & 4.90$\pm$0.53 & 4.55$\pm$0.20 & 55.7 \\
& Gemini-2.5-flash& \third{211.82$\pm$40.30} & 72.70$\pm$1.85 & 5.68$\pm$0.45 & 2.45$\pm$2.00 & 60.0 \\
& Grok-4          & 348.45$\pm$228.02 & \third{74.52$\pm$2.36} & 5.93$\pm$0.78 & 3.36$\pm$0.34 & 61.0 \\
& Qwen-plus       & \second{130.89$\pm$109.30} & \first{75.67$\pm$7.69} & 5.46$\pm$1.13 & 3.84$\pm$0.82 & 54.3 \\
& Convex Optimization     & 346.69$\pm$189.97 & \worst{68.62$\pm$3.90} & \worst{6.54$\pm$4.08} & \worst{2.20$\pm$0.10} & \worst{44.7} \\
\bottomrule
\end{tabular}
}
\end{table*}

\emph{2. Do LLMs possess sufficient domain knowledge to perform effective planning, and what are the limits of their planning capabilities?} \textbf{YES}, agents based on LLMs largely possess sufficient domain knowledge to act as effective high-level planners for AP deployment (see Tables~\ref{tab:b14_apscale}, \ref{tab:b18_apscale}, \ref{tab:b2_apscale} and \ref{tab:b5_apscale}). As generalist foundation models, they encode substantial prior knowledge about wireless propagation, spatial layouts, and
deployment heuristics, enabling them to propose reasonable candidate AP placements across diverse indoor environments. 
However, their planning capability is fundamentally limited by the need to rely on a verifier. As illustrated in Figure~\ref{fig:agentic_flow}, LLM-based planning operates within an agentic loop that repeatedly queries a simulator to  evaluate coverage, interference, and throughput. Since the LLM itself cannot accurately internalize fine-grained physical signal propagation, effective planning requires frequent verifier calls for validation and refinement. This tight dependency leads to substantial inference-time overhead, as each iteration incurs significant computational cost from physical simulation, resulting in slow and resource-intensive planning despite the strong reasoning abilities of the LLM.

\emph{3. Why does diffusion sampling outperform other generative sampling and refinement methods in this setting?} Diffusion sampling outperforms alternative generative refinement methods not due to iterative refinement itself, but because of the implicit smoothing induced by the diffusion process. While weighted sampling with Gaussian proposals or Langevin sampling also perform local refinement guided by the reward, they operate directly on the original reward-induced landscape, which is highly non-convex and fragmented. Consequently, their sampling trajectories are sensitive to local irregularities and easily become trapped in suboptimal modes shown in Figure~\ref{fig:AP_iternative_refinment}. In contrast, diffusion models under VE SDEs apply a Gaussian mollifier to the reward-induced score (see \eqref{eq:derivation_score_function} and Figure~\ref{fig:reward-guided_3d_traj}), progressively smoothing the landscape. This coarse-to-fine sampling process enables more effective exploration of global structure and facilitates convergence toward higher-quality, near-global optima (see Figure~\ref{fig:AP_iternative_refinment}, Tables~ \ref{tab:b14_apscale}, \ref{tab:b18_apscale}, \ref{tab:b2_apscale} and \ref{tab:b5_apscale}). Also, the scaling results are in Figure~\ref{fig:diffusion_scaling}.

\emph{4. What are the main bottlenecks of diffusion sampling for AP deployment tasks?} The primary bottleneck of diffusion sampling in AP deployment lies in the reward function. While a learned general reward exhibits a degree of out-distribution generalization (see Table~\ref{tab:diffusion_rewardnet}), performance degrades significantly when applied to deployment scenarios that differ substantially from the training data. This degradation can be attributed to inaccuracies in the reward approximation. Specifically, if the reward estimation error satisfies $| r(\mathbf{P}, \mathcal{D}) - \hat{r}(\mathbf{P}, \mathcal{D}) | \leq \varepsilon$ in Theorem~\ref{theorem:reward_bottleneck}, the induced error in the score function admits an upper bound of $\mathcal{O}(\beta^{4}\varepsilon^{2})$, and the discrepancy with target distribution is amplified during sampling and leads to suboptimal deployment solutions. However, we observe that incorporating few-shot samples from the target domain substantially improves reward accuracy and restores performance. As the number of samples increases, scaling the reward model leads to consistently higher-quality diffusion-based deployment solutions.

\textbf{Ablation Studies.} We study the influence of the Monte Carlo sample number on diffusion sampling to validate our theoretical analysis. As shown in Table~\ref{tab:diffusion_particles}, increasing the number of particles consistently improves deployment quality across coverage, IOR, TQS, and success rate, while reducing performance variance. This behavior aligns with the convergence rate of the score estimator, which scales as $\mathcal{O}(1/(|\Pi|K))$, where $K$ denotes the number of Monte Carlo samples and $|\Pi|$ accounts for permutation averaging. More accurate score estimates enable the diffusion process to follow a smoother and more reliable optimization direction, thereby requiring fewer diffusion iterations (with less time) to reach the target coverage level. More ablation on temperature shown in Apendix~\ref{append:extra_ablation}.  

\begin{table*}[ht]
\centering
\caption{Performance comparison under different AP deployment scales in Building Level 4. The simulation results are reported as (mean$\pm$std) with color coding indicating performance rankings: the \textit{top three} methods in each metric with different number of APs are highlighted in \colorbox{green!25}{(1st)}, \colorbox{blue!15}{(2nd)}, and \colorbox{gray!20}{(3rd)}. The \worst{worst} performance in each column is underlined. Due to the space limitation, Level 1 (Figure~\ref{tab:b18_apscale}), Level 2 (Figure~\ref{tab:b2_apscale}) and Level 3 (Figure~\ref{tab:b5_apscale}) are presented in Appendix~\ref{append:Simulation Results for different levels}.}
\label{tab:b14_apscale}
\setlength{\tabcolsep}{8pt}
\small
\begin{tabular}{ccccccc}
\toprule
\textbf{APs} & \textbf{Method} 
& \textbf{Runtime (s)$\downarrow$} 
& \textbf{Coverage (\%)$\uparrow$} 
& \textbf{IOR$\downarrow$} 
& \textbf{TQS$\uparrow$} 
& \textbf{Success (\%)$\uparrow$} \\
\midrule

\multirow{10}{*}{3}
& Diffusion       & \first{31.66±27.47} & \third{58.98±0.64} & \third{0.45±1.69} & 9.31±4.27 & \first{82.0} \\
& Weighted sampling (Gaussian)            & \third{71.37±52.31} & 56.91±0.47 & 0.97±1.32 & 6.57±3.66 & \third{78.0} \\
& Weight sampling (Langevin)         & \second{62.01±70.45} & 57.15±0.47 & \second{0.26±0.37} & \first{10.97±3.35} & \second{80.7} \\
& DeepSeek-V3.2    & 195.43±15.33 & 57.35±0.20 & 1.15±1.14 & 8.23±3.65 & 74.0 \\
& GPT-4-turbo      & 597.20±619.10 & 56.92±1.50 & 0.96±1.18 & 9.91±4.41 & 72.7 \\
& Claude-3-Haiku   & 797.81±708.08 & 58.00±6.45 & 1.04±1.64 & \third{9.93±5.64} & 70.7 \\
& Gemini-2.5-flash & 225.15±201.89 & \second{61.69±10.56} & 1.45±2.14 & 7.45±5.83 & 75.3 \\
& Grok-4           & 284.17±220.43 & 57.51±0.37 & \first{0.01±0.23} & \second{10.87±0.06} & 76.0 \\
& Qwen-plus        & \worst{1142.55±903.67} & \first{61.89±12.33} & 1.23±2.13 & 6.57±3.66 & 69.3 \\
& Convex Optimization      & 205.54±26.13 & \worst{44.54±9.09} & \worst{1.54±2.17} & \worst{3.93±2.20} & \worst{58.0} \\
\midrule

\multirow{10}{*}{6}
& Diffusion       & \first{39.91±12.71} & 67.43±1.03 & 4.11±0.51 & \first{16.40±0.21} & \first{74.7} \\
& Weighted sampling (Gaussian)             & 172.62±198.39 & \third{67.74±0.82} & \third{2.49±2.23} & 9.73±6.95 & \third{70.0} \\
& Weight sampling (Langevin)         & \second{128.02±124.46} & \second{67.86±0.87} & \second{1.43±0.50} & \third{13.16±0.59} & \second{72.7} \\
& DeepSeek-V3.2    & 193.73±94.54 & 67.62±1.45 & 5.11±0.91 & 2.87±0.83 & 65.3 \\
& GPT-4-turbo      & 337.70±168.25 & 66.56±0.96 & \first{1.31±0.51} & \second{13.42±5.75} & 64.0 \\
& Claude-3-Haiku   & 333.56±82.36 & 67.59±1.76 & 3.05±1.72 & 6.90±5.23 & 62.0 \\
& Gemini-2.5-flash & \third{140.21±123.65} & 66.49±1.92 & 4.46±0.58 & 3.02±0.81 & 66.7 \\
& Grok-4           & \worst{620.79±266.21} & 66.49±1.92 & 4.67±0.62 & 10.20±0.76 & 67.7 \\
& Qwen-plus        & 430.68±214.13 & \first{68.22±1.45} & \worst{5.32±0.54} & \worst{2.46±1.14} & 61.0 \\
& Convex Optimization      & 267.58±85.57 & \worst{60.67±7.99} & 5.24±3.08 & 3.16±0.59 & \worst{50.7} \\
\midrule

\multirow{10}{*}{8}
& Diffusion       & \first{102.26±71.57} & \second{74.52±1.24} & \first{2.73±0.23} & \first{13.25±2.75} & \first{68.0} \\
& Weighted sampling (Gaussian)             & 346.75±352.49 & 73.31±0.70 & \second{3.80±0.63} & 3.89±2.13 & \third{63.3} \\
& Weight sampling (Langevin)         & 331.34±222.24 & 73.85±0.67 & 4.24±1.87 & \third{5.56±2.24} & \second{66.0} \\
& DeepSeek-V3.2    & 262.15±146.74 & 72.49±1.51 & \worst{6.54±0.23} & 3.68±1.23 & 58.7 \\
& GPT-4-turbo      & 243.29±158.90 & 72.41±1.67 & \third{3.84±1.80} & \second{12.95±8.77} & 57.3 \\
& Claude-3-Haiku   & \worst{689.94±331.97} & 72.90±3.17 & 4.90±0.53 & 4.55±0.20 & 55.7 \\
& Gemini-2.5-flash & \third{211.82±40.30} & 72.70±1.85 & 5.68±0.45 & 2.45±2.00 & 60.0 \\
& Grok-4           & 348.45±228.02 & \third{74.52±2.36} & 5.93±0.78 & 3.36±0.34 & 61.0 \\
& Qwen-plus        & \second{130.89±109.30} & \first{75.67±7.69} & 5.46±1.13 & 3.84±0.82 & 54.3 \\
& Convex Optimization      & 346.69±189.97 & \worst{68.62±3.90} & \worst{6.54±4.08} & \worst{2.20±0.10} & \worst{44.7} \\
\bottomrule
\end{tabular}
\end{table*}

\vspace{-4pt}

\begin{table}[ht]
\centering
\footnotesize
\setlength{\tabcolsep}{2pt}
\renewcommand{\arraystretch}{0.99}
\caption{Diffusion performance under different reward network training regimes. The reward network predicts coverage and provides guidance during diffusion sampling. Runtime (RT), coverage (Cov.), IOR, TQS and success rate (SR) are reported as (mean $\pm$ std).}
\label{tab:diffusion_rewardnet}
\begin{tabular}{@{}lccccc@{}}
\toprule
\textbf{Reward}
& \textbf{RT (s)}
& \textbf{Cov. (\%)}
& \textbf{IOR}
& \textbf{TQS}
& \textbf{SR (\%)} \\
\midrule
In-distribution
& $30.3{\pm}11.6$
& $92.5{\pm}3.2$
& $0.9{\pm}0.1$
& $51.1{\pm}0.2$
& 92 \\
Few-shot
& $172.5{\pm}71.2$
& $68.4{\pm}0.8$
& $1.9{\pm}1.3$
& $39.2{\pm}1.4$
& 74 \\
Zero-shot
& $352.6{\pm}54.8$
& $36.7{\pm}4.9$
& $4.9{\pm}2.1$
& $30.5{\pm}3.2$
& 31 \\
\bottomrule
\end{tabular}
\end{table}

\begin{table}[ht]
\centering
\footnotesize
\setlength{\tabcolsep}{2pt}
\renewcommand{\arraystretch}{0.99}
\caption{Effect of particle number on diffusion-based optimization performance.
Runtime (RT), coverage (Cov.), IOR, TQS and success rate (SR) are reported as (mean $\pm$ std).}
\label{tab:diffusion_particles}
\begin{tabular}{@{}cccccc@{}}
\toprule
\textbf{\# Particles}
& \textbf{RT (s)}
& \textbf{Cov. (\%)}
& \textbf{IOR}
& \textbf{TQS}
& \textbf{SR (\%)} \\
\midrule
5
& $163.3{\pm}11.9$
& $70.7{\pm}3.3$
& $1.1{\pm}0.6$
& $35.9{\pm}2.1$
& 89 \\
10
& $60.8{\pm}11.0$
& $71.4{\pm}3.6$
& $1.4{\pm}0.2$
& $34.5{\pm}3.0$
& 93 \\
20
& $43.3{\pm}14.2$
& $75.2{\pm}2.4$
& $1.8{\pm}0.3$
& $36.8{\pm}2.7$
& 98 \\
\bottomrule
\end{tabular}
\end{table}



\section{Conclusion and Limitations}
We studied indoor AP deployment through the lens of modern AI and systematically benchmarking the tasks over a wide range of models. By reframing AP planning as a generative inference problem, our work opens a new research direction for scalable and domain-agnostic wireless network design. Our results show that diffusion sampling consistently outperforms alternatives in highly non-convex indoor environments. These gains stem from the diffusion process smoothing the reward-induced score landscape, enabling effective sampling over fragmented deployment spaces. By introducing a large-scale real-world dataset, we further demonstrate that a unified reward function can generalize across diverse indoor layouts, highlighting diffusion-based inference as a scalable foundation for indoor AP planning.

\textbf{Limitations. }Our approach depends on the quality of the reward function and currently focuses on static indoor environments. Future work may explore richer reward formulations, outdoor environments, dynamic deployment scenarios for broader wireless planning tasks.

\section*{Impact Statement}
This paper presents work whose goal is to advance the intersection field of AI and wireless communication. There are many potential societal consequences of our work, none which we feel must be specifically highlighted here.

\bibliography{example_paper}
\bibliographystyle{icml2026}

\newpage
\appendix
\onecolumn
\section{You \emph{can} have an appendix here.}
\section*{Notation}

\begin{center}
\begin{tabular}{cc}
\toprule 
Notations & Meaning   \\ 
\midrule 
$B$ & system bandwidth \\ 
$\mathcal{D}$ & bounded region \\ 
$I$ & co-channel interference \\ 
$N$ & number of APs \\ 
$\mathbf{P}$ & deployment configuration \\
$R$ & coverage function \\
$T$ & service feasibility \\
$T_{\text{min}}$ & minimal requirement of service \\
$\mathcal{Z}$ & partition function \\
$W$ & Brownian motion \\ 
$p$ & probability \\ 
$q$ &  proposal distribution \\ 
$r$ &  reward function \\ 
$\mathbf{p}_n$ & the location of the $n-$th AP \\
$\beta$ & temperature factor \\ 
$\sigma_t$ & noise scale at diffusion time step $t$ \\ 
$\omega$ & importance weight \\ 
$\pi$ & permutation group action \\ 
$\xi$ & co-channel interference threshold \\ 
$\Pi$ & permutation invariance group \\ 

\bottomrule 
\end{tabular}
\end{center}

\clearpage
\section{Quick Guidance for Appendix}

This appendix provides a structured and self-contained supplement to the main paper, consolidating the technical foundations, theoretical analysis, algorithmic details, data generation protocol, and extended experimental results. Its organization is designed to guide readers progressively from modeling assumptions to theoretical guarantees and empirical validation.

\textbf{Foundational Definitions.}
Section~\ref{append:diffusion} introduces the diffusion-based generative framework adopted in this work, including the forward variance-exploding SDE and the corresponding reverse-time dynamics.
Section~\ref{append:permutation_group} formalizes the permutation invariance induced by indistinguishable access point (AP) identities, which plays a central role in score estimation, variance reduction, and sample efficiency.

\textbf{Theoretical Analysis.}
Theoretical guarantees are presented next.
Lemma~\ref{lemma:score_concentration} establishes a concentration bound for Monte Carlo score estimation, showing an improved $\mathcal{O}(1/\sqrt{|\Pi|K})$ convergence rate enabled by permutation symmetrization.
Theorems~\ref{theorem:reward_bottleneck} and~\ref{thm:reward_score_gap} further characterize how reward approximation errors propagate through the score function and affect the final sampling distribution, identifying reward estimation as the dominant source of error in reward-guided diffusion.

\textbf{Reward Design and Optimization.}
Section~\ref{append:reward_design} details the differentiable, geometry-aware reward construction, including floor-plan encoding, location-conditioned feature querying, and permutation-invariant aggregation.
This section connects the theoretical analysis with practical optimization by enabling both gradient-based methods and diffusion-based sampling.

\textbf{Algorithms and Implementation Details.}
Algorithmic procedures are specified in Section~\ref{append:code_and_implementation}, which provides pseudo-code for weighted sampling, gradient-accelerated (Langevin) sampling, and diffusion-based optimization.
This section clarifies how the theoretical score formulations are instantiated in practice and how permutation invariance is enforced during sampling.

\textbf{Data Generation and Benchmark Construction.}
Section~\ref{append:data_generation} describes the data generation pipeline used to construct the large-scale indoor deployment benchmark.
It details the use of a physics-based 3D ray-tracing simulator, the building layout corpus, wireless configuration, and supervision signals, ensuring that training and evaluation are grounded in realistic propagation and interference conditions.

\textbf{Extended Experimental Results and Visualization.}
Section~\ref{append:more_experimental_result} reports extended simulation results across multiple building complexity levels, together with ablation studies and qualitative visualizations.
The intermediate reasoning and iterative optimization trajectories illustrate how different baselines evolve during optimization and how diffusion-based planning progressively refines AP deployments in complex indoor environments.

Overall, the appendix is organized to move systematically from definitions, to theory, to algorithms, to data, and finally to empirical evidence. It complements the main paper by providing depth and rigor without introducing new assumptions or altering the core problem formulation.

\clearpage
\section{Key Technical Definitions}

\subsection{Diffusion-Based Generative Model} \label{append:diffusion}

Diffusion models provide a powerful framework for approximating complex, high-dimensional probability distributions. In this work, we adopt the continuous-time variance exploding (VE) diffusion formulation \citep{akhound2024iterated}. The forward diffusion process gradually perturbs a data sample $\mathbf{P}_0$ into Gaussian noise over the time horizon $t \in [0, T]$, and is governed by the following stochastic differential equation (SDE):
\begin{equation}
d\mathbf{P}_t = \sqrt{\frac{d \sigma_t^2}{dt}} \, dW_t,
\end{equation}
where $\sigma_t$ denotes a predefined noise schedule and $W_t$ is a standard Brownian motion.

A key property of diffusion models is that the forward process admits a well-defined reverse-time dynamics. Specifically, the reverse-time SDE that transforms noise back into data samples is given by
\begin{align}
\label{eq:reverse_SDE}
d\mathbf{P}_t = \left[ - \frac{d \sigma_t^2}{dt} \nabla_{\mathbf{P}} \log p(\mathbf{P}_t|\mathcal{D}) \right] dt
+ \sqrt{\frac{d \sigma_t^2}{dt}} \, dW_t .
\end{align}
Solving the reverse SDE requires access to the score function $\nabla_{\mathbf{P}} \log p(\mathbf{P}_t)$, which represents the gradient of the log-density of the perturbed data distribution. In standard diffusion models, this score function is learned via neural network training and estimated through Monte Carlo sampling, as shown in \eqref{eq:score_function}.

In our setting, however, the score function can be obtained analytically without additional training once the reward function is known. This enables us to perform diffusion-based sampling directly by simulating the reverse SDE, starting from an initial noise sample $\mathbf{P}_T \sim \mathcal{N}(\boldsymbol{0}, \boldsymbol{I})$.

\subsection{Permutation Group and Invariance} \label{append:permutation_group}

Consider a deployment of $N$ access points (APs) represented by the unordered set
$\mathbf{P} = \{\mathbf{p}_1, \dots, \mathbf{p}_N\}$, where each $\mathbf{p}_i \in \mathbb{R}^3$
denotes the spatial location of the $i$-th AP. Since APs are indistinguishable, any reordering of their indices corresponds to the same physical deployment.
Formally, let $\Pi$ denote the symmetric group of degree $N$, consisting of all permutations over the index set $\{1,\dots,N\}$.
For any permutation $\pi \in \Pi$, the induced group action on the deployment $\mathbf{P}$ is defined as
\[
\pi(\mathbf{P}) = \{\mathbf{p}_{\pi(1)}, \dots, \mathbf{p}_{\pi(N)}\}.
\]

A function $f(\mathbf{P}, \mathcal{D})$ defined on AP deployments and environmental context $\mathcal{D}$
is said to be \emph{permutation invariant} if it remains unchanged under any such permutation, i.e.,
\[
f(\mathbf{P}, \mathcal{D}) = f(\pi(\mathbf{P}), \mathcal{D}), \quad \forall\, \pi \in \Pi.
\]
In particular, the reward function $r(\mathbf{P}, \mathcal{D})$ considered in this work satisfies permutation invariance,
reflecting the indistinguishability of AP identities.

This invariance induces an equivalence class structure over the space of deployments, where all permutations of a given configuration correspond to the same physical state and yield identical rewards.
Exploiting this symmetry allows us to treat permuted deployments as equivalent samples, thereby improving sample efficiency and reducing variance in both the normalization constant estimation in \eqref{eq:normalization}
and the score function estimation in \eqref{eq:score_function_estimation}.

\clearpage
\section{Theoretical Analysis}

\textbf{Derivation of Exact Score Function.}

\begin{equation}\label{eq:derivation_score_function}
\begin{aligned}
\nabla_{\mathbf{P}} \log p_t (\mathbf{P}_t|\mathcal{D})
&= \frac{\big((\nabla_{\mathbf{P}} p_{0}) * \mathcal{N}(0, \sigma_t^2I) \big)(\mathbf{P}_t|\mathcal{D})}{p_t (\mathbf{P}_t|\mathcal{D})} \\
&= \frac{\mathbb{E}_{\mathbf{P}_{0|t} \sim \mathcal{N}(\mathbf{P}_t, \sigma_t^2I)}[\nabla_{\mathbf{P}} p_{0}(\mathbf{P}_{0|t})]}
        {\mathbb{E}_{\mathbf{P}_{0|t} \sim \mathcal{N}(\mathbf{P}_t, \sigma_t^2I)}[p_{0}(\mathbf{P}_{0|t})]} \\
&= \frac{\mathbb{E}_{\mathbf{P}_{0|t} \sim \mathcal{N}(\mathbf{P}_t, \sigma_t^2I)}\!\left[\frac{\nabla_{\mathbf{P}} e^{\beta r(\mathbf{P}_{0|t}, \mathcal{D})}}{\cancel{\mathcal{Z}}}\right]}
        {\mathbb{E}_{\mathbf{P}_{0|t} \sim \mathcal{N}(\mathbf{P}_t, \sigma_t^2I)}\!\left[\frac{e^{\beta r(\mathbf{P}_{0|t}, \mathcal{D})}}{\cancel{\mathcal{Z}}}\right]} \\
&= \frac{\mathbb{E}_{\mathbf{P}_{0|t} \sim \mathcal{N}(\mathbf{P}_t, \sigma_t^2I)}[\nabla_{\mathbf{P}} e^{\beta r(\mathbf{P}_{0|t}, \mathcal{D})}]}
        {\mathbb{E}_{\mathbf{P}_{0|t} \sim \mathcal{N}(\mathbf{P}_t, \sigma_t^2I)}[e^{\beta r(\mathbf{P}_{0|t}, \mathcal{D})}]}\, .
\end{aligned}
\end{equation}

Here, $*$ represents convolution with an isotropic Gaussian mollifier corresponding to the variance-exploding noise with variance $\sigma_t^2I$. The first line follows directly from fundamental properties of convolution operators~\cite{brezis2011functional}, while the remaining steps are derived using standard probabilistic representations. \textit{Under the VE SDE, the Gaussian mollifier naturally induces a coarse-to-fine structure over the reward landscape: at large noise levels, convolution with a wide Gaussian smooths out local irregularities and fragmented modes for computing the score function, while progressively smaller noise sharpens the landscape and recovers fine-grained structure. This multiscale smoothing effect enables diffusion sampling to effectively navigate highly non-convex and fragmented objectives (as shown in Figure~\ref{fig:reward-guided_3d_traj}). } 

\begin{figure*}[ht]
    \centering
    \includegraphics[width=0.88\linewidth]{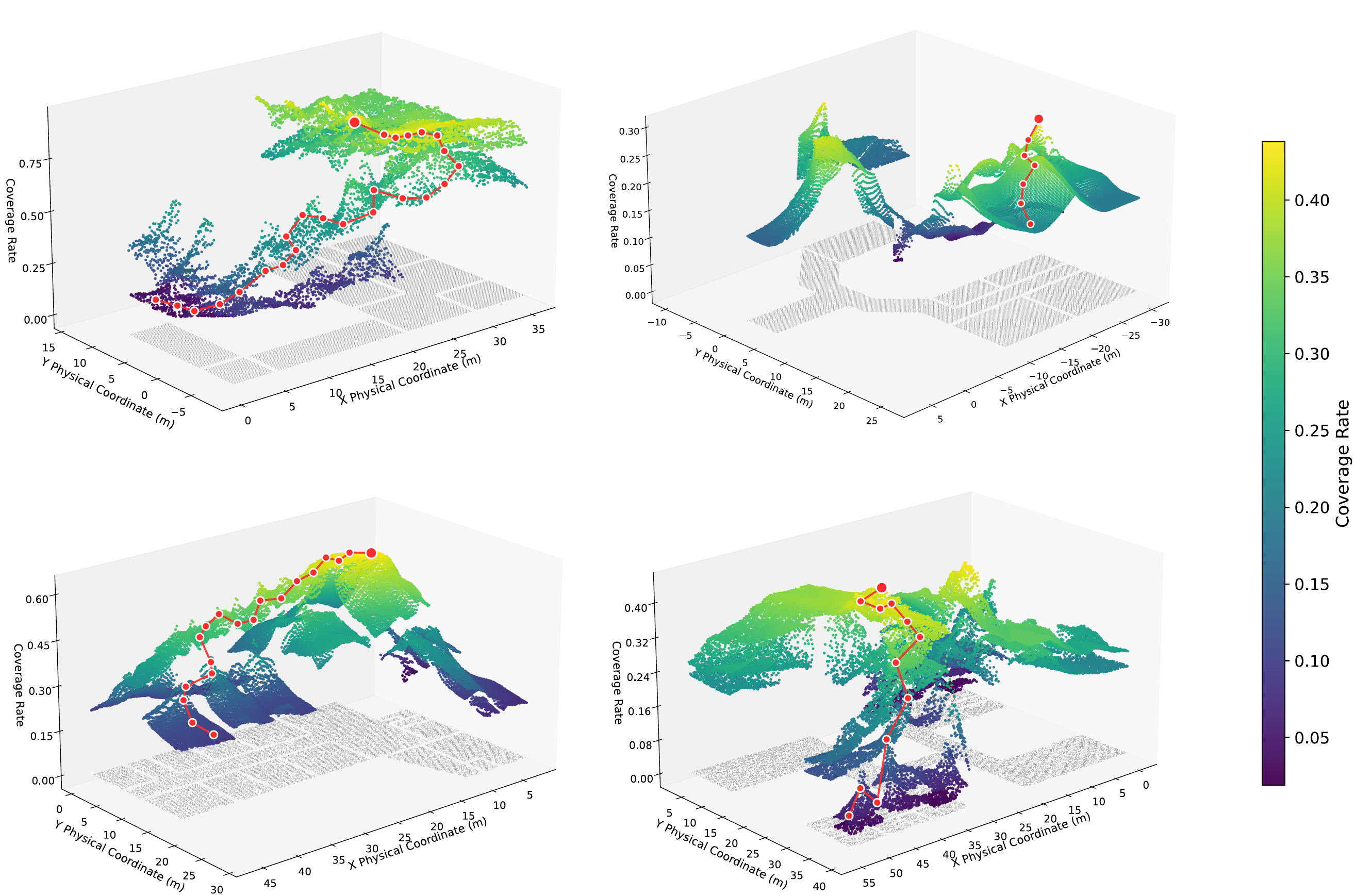}
    \caption{Visualization of reward-guided diffusion trajectories. The shaded regions on the ground plane depict the indoor floor plan, while the colored surfaces represent the reward landscape over AP deployment space. The red trajectories illustrate reverse-time SDE sampling paths guided by the reward-induced score. Due to Gaussian mollification under the VE SDE, the score function computed with a coarse-to-fine structure: large noise levels smooth non-convex and fragmented reward regions, enabling globally coherent motion, while progressively smaller noise refines the trajectory toward high-reward optima. In contrast to convex optimization in Figure~\ref{fig:AP_iternative_refinment}, which is prone to local traps, the injected VE noise allows diffusion sampling to overcome disconnected modes and navigate complex reward landscapes effectively.}
    \label{fig:reward-guided_3d_traj}
\end{figure*}

\begin{lemma}[Girsanov Theorem for Drift Perturbation \citep{shiryaev2016probability}] \label{lemma:girsanov}
Let $\{\mathbf{P}_t\}_{t\in[0,T]}$ be a stochastic process governed by the reverse-time SDE
\[
d\mathbf{P}_t = b(\mathbf{P}_t, t)\,dt + \eta_t\,dW_t,
\]
and let $\{\hat{\mathbf{P}}_t\}_{t\in[0,T]}$ follow
\[
d\hat{\mathbf{P}}_t = \hat{b}(\hat{\mathbf{P}}_t, t)\,dt + \eta_t\,dW_t,
\]
where both processes share the same diffusion coefficient $\eta_t$ and initial distribution.
Assume that the Novikov condition holds for the drift difference.
Then, the Kullback–Leibler divergence between the induced path measures satisfies
\[
D_{\mathrm{KL}}\big(p_{\hat{b}} \,\|\, p_{b}\big)
=
\frac{1}{2}
\mathbb{E}_{p_{\hat{b}}}
\left[
\int_0^T
\left\|
\eta_t^{-1}
\big(
b(\mathbf{P}_t,t)-\hat{b}(\mathbf{P}_t,t)
\big)
\right\|^2
dt
\right].
\]
\end{lemma}

\begin{lemma}
\label{lemma:score_concentration}
Assume that $e^{\beta r(\mathbf{P}, \mathcal{D})}$ and
$\|\nabla_{\mathbf{P}} e^{\beta r(\mathbf{P}, \mathcal{D})}\|$
are sub-Gaussian random variables and $\mathbf{P}$ is permutation invariant.
Then there exists a constant $c(\mathbf{P}_t, \mathcal{D})$ such that,
for any $\delta \in (0,1)$, with probability at least $1-\delta$,
\begin{equation}
\label{eq:score_concentration}
\left\|
\widehat{\nabla_{\mathbf{P}} \log p_t}(\mathbf{P}_t \mid \mathcal{D})
-
\nabla_{\mathbf{P}} \log p_t (\mathbf{P}_t \mid \mathcal{D})
\right\|
\leq
\frac{
c(\mathbf{P}_t, \mathcal{D}) \log(1/\delta)
}{
\sqrt{|\Pi|K}
}.
\end{equation}
Here, $\widehat{\nabla_{\mathbf{P}} \log p_t}(\mathbf{P}_t \mid \mathcal{D})$ is estimated as shown in \eqref{eq:score_function_estimation}, and $K$ is the number in Monte Carlo estimation. 
\end{lemma}

\textit{Proof Sketch.}
The proof follows directly from the concentration analysis of Monte Carlo score estimators
in \citep{akhound2024iterated}.
Under the sub-Gaussian assumptions on $e^{\beta r(\mathbf{P}, \mathcal{D})}$ and its gradient,
the Monte Carlo estimator of the score function admits a concentration bound.
Applying Hoeffding's inequality to each component of the estimator and a union bound yields
the $\mathcal{O}(1/\sqrt{K})$ deviation bound in ~\eqref{eq:score_concentration}.
Exploiting permutation invariance, the estimator can be symmetrized by averaging over the permutation group,
which reduces variance without introducing bias.
As a result, the deviation bound improves to $\mathcal{O}(1/\sqrt{|\Pi|K})$,
where $|\Pi|$ denotes the cardinality of the permutation group.

\begin{proofwithframe}
We decompose the discrepancy between the generated distribution
$\hat{p}_0(\mathbf{P}, \mathcal{D})$ and the target distribution
$p(\mathbf{P}, \mathcal{D})$ into a modeling error induced by reward approximation
and a statistical error arising from Monte Carlo estimation.

\paragraph{Step 1: Path-space KL induced by reward approximation.}
Let $p$ denote the path measure induced by the reverse-time SDE with the
true drift $b(\mathbf{P}_t,t)$ induced by true score function $-\frac{d\sigma_t^2}{dt}\log p_t (\mathbf{P}|\mathcal{D})$, and let $p_{0}$ denote the path measure 
corresponding to the reverse-time SDE driven by the drift $\hat b(\mathbf{P}_t,t)$
constructed from the estimated reward function.
By Lemma~\ref{lemma:girsanov}, we have
\begin{equation}
\label{eq:path_kl_reward}
D_{\mathrm{KL}}\big(p \,\|\, \hat{p}_0\big)
=
\frac{1}{2}
\mathbb{E}_{p}
\left[
\int_0^T
\left\|
(\frac{d\sigma_t}{dt})^{-1}
\big(
b(\mathbf{P}_t,t)-\hat b(\mathbf{P}_t,t)
\big)
\right\|^2
dt
\right].
\end{equation}
Since the diffusion coefficient $\sigma_t$ are bounded in VE SDE, the above expression can be
upper bounded by
\begin{equation}
\label{eq:reward_gap_bound}
D_{\mathrm{KL}}\big(p \,\|\, \hat{p}_0\big)
\le
c_1
\int_0^T
\|
b(\mathbf{P}_t,t)-\hat b(\mathbf{P}_t,t)
\|^2
dt,
\end{equation}
where $c_1>0$ is a constant. Here, we introduce the intermediate variable $\tilde{b}$ which is the exact score without Monte Carlo estimation. Then, \eqref{eq:reward_gap_bound} becomes 
\begin{equation}
\begin{split}
    D_{\mathrm{KL}}\big(p \,\|\, \hat{p}_0\big)
& \le
c_1
\int_0^T
\|
b(\mathbf{P}_t,t)- \tilde b(\mathbf{P}_t,t)  + \tilde b(\mathbf{P}_t,t)-\hat b(\mathbf{P}_t,t)
\|^2
dt \\
& \leq c_1
\int_0^T
\underbrace{2\|b(\mathbf{P}_t,t)- \tilde b(\mathbf{P}_t,t) \|^2}_{\text{induced by the reward gap}} +  \underbrace{2\| \tilde b(\mathbf{P}_t,t)-\hat b(\mathbf{P}_t,t)\|^2}_{\text{induced by the Monte Carlo estimation}}
dt.
\end{split}
\end{equation}

\paragraph{Step 2: Statistical error from Monte Carlo score estimation.}
In practice, the drift $\tilde b(\mathbf{P}_t,t)$ is not available in closed form
and is approximated via Monte Carlo estimation of the score function
$\nabla_{\mathbf{P}} \log \tilde{p}_t(\mathbf{P}_t \mid \mathcal{D})$
as in Eq.~\eqref{eq:score_function_estimation}.
Let $\hat b(\mathbf{P}_t,t)$ denote the resulting empirical drift.

By Lemma~\ref{lemma:score_concentration}, with probability at least $1-2\delta$,
the deviation between the empirical drift $\widehat b$ and its population
counterpart $\hat b$ satisfies
\begin{equation}
\label{eq:mc_drift_error}
\|
\tilde{b}(\mathbf{P}_t,t) - \hat b(\mathbf{P}_t,t)
\|
\le
\frac{c_2 \log(1/\delta)}{\sqrt{|\Pi|K}},
\end{equation}
for some constant $c_2>0$.
This term captures the statistical error due to finite Monte Carlo sampling
and permutation symmetrization.

\paragraph{Step 3: Combining modeling and statistical errors.}
Applying the data processing inequality to the terminal distributions and
combining equations~\ref{eq:reward_gap_bound} and \ref{eq:mc_drift_error}, we obtain
that, with probability at least $1-2\delta$,
\begin{equation}
D_{\mathrm{KL}}\big(p(\mathbf{P}|\mathcal{D}) \,\|\, \hat{p}_0(\mathbf{P}|\mathcal{D})\big)
\le
\underbrace{
c_1 \int_0^T
2\|
b(\mathbf{P}_t,t)-\tilde b(\mathbf{P}_t,t)
\|^2 dt
}_{\text{error from reward estimation}}
+
\underbrace{
\frac{2c_2T\log(1/\delta)}{|\Pi|K}
}_{\text{error from Monte Carlo estimation}},
\end{equation}
which completes the proof.
\end{proofwithframe}

To further analyse the gap induced by reward estimation, we give the Theorem~\ref{thm:reward_score_gap} as follows. 

\begin{claimwithframe}
\label{thm:reward_score_gap}
Let the target score function at time $t$ be defined as
\begin{equation}
\label{eq:true_score}
\nabla_{\mathbf{P}} \log p_t (\mathbf{P}_t \mid \mathcal{D})
=
\frac{
\mathbb{E}_{\mathbf{P}_{0|t} \sim \mathcal{N}(\mathbf{P}_t, \sigma_t^2 I)}
\!\left[
\nabla_{\mathbf{P}} e^{\beta r(\mathbf{P}_{0|t}, \mathcal{D})}
\right]
}{
\mathbb{E}_{\mathbf{P}_{0|t} \sim \mathcal{N}(\mathbf{P}_t, \sigma_t^2 I)}
\!\left[
e^{\beta r(\mathbf{P}_{0|t}, \mathcal{D})}
\right]
}.
\end{equation}
Let $\hat r(\mathbf{P},\mathcal{D})$ be an approximate reward function satisfying
\begin{equation}
\label{eq:reward_gap}
|r(\mathbf{P},\mathcal{D}) - \hat r(\mathbf{P},\mathcal{D})|
\le \varepsilon,
\quad \forall\,\mathbf{P},
\end{equation}
and assume that both $r$ and $\hat r$ are differentiable with gradients bounded as
\[
\|\nabla_{\mathbf{P}} r(\mathbf{P},\mathcal{D})\|,
\;
\|\nabla_{\mathbf{P}} \hat r(\mathbf{P},\mathcal{D})\|
\le G.
\]
Let $\hat s(\mathbf{P}_t)$ denote the score induced by $\hat r$ via \eqref{eq:true_score}.
Then there exist constants $C_1,C_2>0$ such that
\begin{equation}
\label{eq:score_gap_bound}
\left\|
\hat s(\mathbf{P}_t)
-
\nabla_{\mathbf{P}} \log p_t (\mathbf{P}_t \mid \mathcal{D})
\right\|
\le
C_1 \beta \varepsilon
+
C_2 \beta^2 \varepsilon .
\end{equation}
In particular, the score deviation satisfies
\[
\left\|
\hat s - s
\right\|^2
=
\mathcal{O}(\beta^4 \varepsilon^2).
\]
\end{claimwithframe}

\begin{proofwithframe}
    Recall that the score function can be rewritten using the identity
$\nabla_{\mathbf{P}} e^{\beta r} = \beta e^{\beta r} \nabla_{\mathbf{P}} r$.
Thus,
\begin{equation}
\label{eq:score_rewrite}
\nabla_{\mathbf{P}} \log p_t (\mathbf{P}_t \mid \mathcal{D})
=
\beta
\frac{
\mathbb{E}\!\left[
e^{\beta r(\mathbf{P}_{0|t},\mathcal{D})}
\nabla_{\mathbf{P}} r(\mathbf{P}_{0|t},\mathcal{D})
\right]
}{
\mathbb{E}\!\left[
e^{\beta r(\mathbf{P}_{0|t},\mathcal{D})}
\right]
}
=
\beta\,\mathbb{E}_{\pi_r}
\!\left[
\nabla_{\mathbf{P}} r
\right],
\end{equation}
where $\pi_r$ is the probability measure defined by
\[
\pi_r(d\mathbf{P}_{0|t})
\propto
e^{\beta r(\mathbf{P}_{0|t},\mathcal{D})}
\mathcal{N}(\mathbf{P}_t,\sigma_t^2 I)\,d\mathbf{P}_{0|t}.
\]
Similarly, the approximate score induced by $\hat r$ is given by
\[
\hat s(\mathbf{P}_t)
=
\beta\,\mathbb{E}_{\pi_{\hat r}}
\!\left[
\nabla_{\mathbf{P}} \hat r
\right].
\]

We decompose the score deviation as
\begin{align}
\label{eq:score_decomp}
\|\hat s - s\|
&=
\beta
\left\|
\mathbb{E}_{\pi_{\hat r}}[\nabla_{\mathbf{P}} \hat r]
-
\mathbb{E}_{\pi_r}[\nabla_{\mathbf{P}} r]
\right\| \\
&\le
\beta
\left\|
\mathbb{E}_{\pi_{\hat r}}
\!\left[
\nabla_{\mathbf{P}} \hat r - \nabla_{\mathbf{P}} r
\right]
\right\|
+
\beta
\left\|
\mathbb{E}_{\pi_{\hat r}}[\nabla_{\mathbf{P}} r]
-
\mathbb{E}_{\pi_r}[\nabla_{\mathbf{P}} r]
\right\|.
\nonumber
\end{align}

We first bound the gradient discrepancy term.
By the uniform reward gap assumption
$|r(\mathbf{P},\mathcal{D})-\hat r(\mathbf{P},\mathcal{D})|\le\varepsilon$
and differentiability of both functions,
there exists a constant $C_1>0$ such that
\begin{equation}
\left\|
\nabla_{\mathbf{P}} \hat r(\mathbf{P},\mathcal{D})
-
\nabla_{\mathbf{P}} r(\mathbf{P},\mathcal{D})
\right\|
\le
C_1 \varepsilon,
\end{equation}
which implies
\[
\left\|
\mathbb{E}_{\pi_{\hat r}}
\!\left[
\nabla_{\mathbf{P}} \hat r - \nabla_{\mathbf{P}} r
\right]
\right\|
\le
C_1 \varepsilon.
\]

Next, we bound the distribution shift term.
The reward gap implies
\[
e^{-\beta\varepsilon}
\le
\frac{e^{\beta \hat r(\mathbf{P},\mathcal{D})}}{e^{\beta r(\mathbf{P},\mathcal{D})}}
\le
e^{\beta\varepsilon},
\]
which yields a stability bound between the tilted measures
$\pi_r$ and $\pi_{\hat r}$.
In particular, their total variation distance satisfies
\[
\mathrm{TV}(\pi_r,\pi_{\hat r}) \le C \beta \varepsilon
\]
for some constant $C>0$.
Using the boundedness assumption
$\|\nabla_{\mathbf{P}} r\|\le G$,
we obtain
\begin{equation}
\left\|
\mathbb{E}_{\pi_{\hat r}}[\nabla_{\mathbf{P}} r]
-
\mathbb{E}_{\pi_r}[\nabla_{\mathbf{P}} r]
\right\|
\le
2G\,\mathrm{TV}(\pi_r,\pi_{\hat r})
\le
C_2 \beta \varepsilon,
\end{equation}
for some constant $C_2>0$.

Substituting the above bounds into \eqref{eq:score_decomp}
yields
\[
\|\hat s - s\|^2
\le
C_1^2 \beta^2 \varepsilon^2
+
C_2^2 \beta^4 \varepsilon^2 + 2C_1C_2\beta^3\varepsilon^2,
\]
which completes the proof.
\end{proofwithframe}

\clearpage
\section{Reward Design and Architecture}
\label{append:reward_design}

\begin{figure}[ht]
    \centering
    \includegraphics[width=0.95\linewidth]{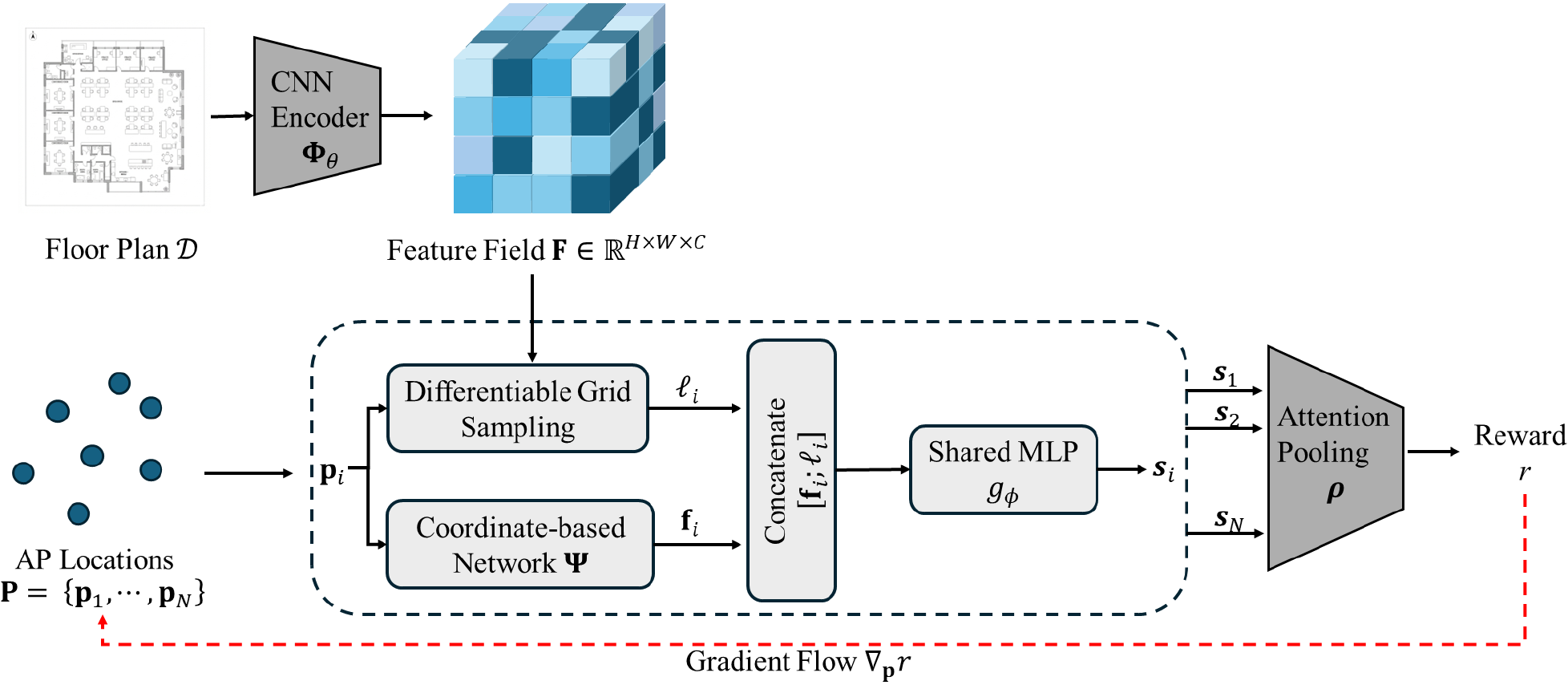}
    \caption{Framework for reward design.}
    \label{fig:reward}
\end{figure}
We design the reward function to be a differentiable, geometry-aware functional that evaluates the quality of an AP deployment conditioned on the building layout.
Let $\mathcal{D}$ denote a building floor plan and
$\mathbf{P} = \{\mathbf{p}_1,\dots,\mathbf{p}_N\}$ the deployment of $N$ access points, where $\mathbf{p}_i \in \mathbb{R}^3$ represents the spatial location of the $i$-th AP.
The reward function is defined as
\begin{equation}
r(\mathbf{P}, \mathcal{D}) = \mathcal{R}\big( \{\mathbf{p}_i\}_{i=1}^N , \mathcal{D} \big),
\end{equation}
where $\mathcal{R}$ is permutation invariant with respect to AP indices.

\paragraph{Floor Plan Encoding.}
The floor plan $\mathcal{D}$ is first rasterized into a multi-channel spatial representation capturing geometric and material properties, such as walls, free space, and obstacles.
A fully convolutional encoder $\Phi_\theta$ maps the floor plan to a dense feature field
\begin{equation}
\mathbf{F} = \Phi_\theta(\mathcal{D}), \quad \mathbf{F} \in \mathbb{R}^{H \times W \times C},
\end{equation}
where $\mathbf{F}(x,y,z)$ encodes local geometric information at spatial location $(x,y,z)$.
The encoder is shared across all floor plans, enabling generalization to unseen layouts.

\paragraph{Location-Conditioned Feature Query.}
Given an AP location $\mathbf{p}_i = (x_i, y_i,z_i)$, we obtain a local geometric descriptor by querying the feature field $\mathbf{F}$ using differentiable bilinear interpolation,
\begin{equation}
\mathbf{f}_i = \mathbf{F}(x_i, y_i,z_i),
\end{equation}
implemented via grid sampling.
This operation preserves differentiability with respect to the AP location.

\paragraph{Continuous Location Embedding.}
To model the continuous dependence of the reward on AP locations, each $\mathbf{p}_i$ is independently embedded using a coordinate-based network $\Psi$,
\begin{equation}
\boldsymbol{\ell}_i = \Psi(\mathbf{p}_i).
\end{equation}
This embedding provides a smooth and expressive representation of spatial coordinates.

\paragraph{Reward Aggregation.}
The local geometric feature $\mathbf{f}_i$ and the location embedding $\boldsymbol{\ell}_i$ are concatenated and passed through a shared MLP $g_\phi$ to produce a per-AP contribution,
\begin{equation}
s_i = g_\phi([\mathbf{f}_i, \boldsymbol{\ell}_i]).
\end{equation}
The overall reward is then obtained via a permutation-invariant aggregation operator,
\begin{equation}
r(\mathbf{P}, \mathcal{D}) = \rho\left( \{ s_i \}_{i=1}^N \right),
\end{equation}
where $\rho$ is an attention-based pooling function.

\paragraph{Differentiability and Optimization.}
All components of the reward function are differentiable with respect to the AP locations.
Consequently, the reward gradient $\nabla_{\mathbf{P}} r(\mathbf{P}, \mathcal{D})$ or $\nabla_{\mathbf{P}} e^{\beta r(\mathbf{P}, \mathcal{D})}$
can be computed via automatic differentiation and directly used for convex optimization or diffusion sampling. 

\clearpage
\section{Pseudo Code and Algorithm Implementation} \label{append:code_and_implementation}





\begin{algorithm}[H]
\caption{Weighted Sampling (a.k.a. Sequential Monte Carlo (SMC)) with Gaussian Proposal}
\label{alg:smc_gaussian}
\begin{algorithmic}
\REQUIRE Region $\mathcal{D}$, reward function $r(\mathbf{P},\mathcal{D})$, temperature $\beta$, proposal variance $\sigma^2$, number of particles $K$, iterations $T$, threshold $K_{thresh}$ (e.g., $K/2$)
\ENSURE Weighted samples $\{\mathbf{P}_T^{(k)}, w_T^{(k)}\}_{k=1}^K$ approximating $p(\mathbf{P}|\mathcal{D})$

\STATE Initialize particles $\{\mathbf{P}_0^{(k)}\}_{k=1}^K \sim \mathcal{N}(\mathbf{0}, I)$
\STATE Initialize weights $w_0^{(k)} = \frac{1}{K}$ for all $k = 1, \dots, K$

\FOR{$t = 1$ \TO \ $T$}
    \FOR{$k = 1$ \TO \ $K$}
        \STATE \textbf{1. Gaussian Random Walk Proposal:}
        \STATE Sample $\mathbf{P}_t^{(k)} \sim \mathcal{N}(\mathbf{P}_{t-1}^{(k)}, \sigma^2 I)$
        
        \STATE \textbf{2. Incremental Weight Update:}
        \STATE Compute permutation-invariant score:
        \[
        \alpha_t^{(k)} = \frac{1}{|\Pi|} \sum_{\pi \in \Pi} \exp\!\big(\beta r(\pi(\mathbf{P}_t^{(k)}),\mathcal{D})\big)
        \]
        \STATE Update unnormalized weight: $\tilde{w}_t^{(k)} = w_{t-1}^{(k)} \cdot \alpha_t^{(k)}$
    \ENDFOR

    \STATE \textbf{3. Weight Normalization:}
    \[
    W_t = \sum_{j=1}^K \tilde{w}_t^{(j)}, \quad w_t^{(k)} = \frac{\tilde w_t^{(k)}}{W_t}
    \]

    \STATE \textbf{4. Adaptive Resampling (Prevents Weight Collapse):}
    \STATE Calculate Effective Sample Size: $N_{eff} = \left( \sum_{k=1}^K (w_t^{(k)})^2 \right)^{-1}$
    \IF{$N_{eff} < K_{thresh}$} 
        \STATE Resample $K$ particles from $\{\mathbf{P}_t^{(k)}\}_{k=1}^K$ according to weights $\{w_t^{(k)}\}_{k=1}^K$
        \STATE Reset weights $w_t^{(k)} = \frac{1}{K}$ for all $k = 1, \dots, K$
    \ENDIF
\ENDFOR

\STATE \textbf{return} $\{\mathbf{P}_T^{(k)}, w_T^{(k)}\}_{k=1}^K$
\end{algorithmic}
\end{algorithm}

\begin{algorithm}[H]
\caption{Weighted Sampling (SMC) with Gradient-Accelerated Langevin Proposals}
\label{alg:smc_langevin}
\begin{algorithmic}
\REQUIRE Region $\mathcal{D}$, reward function $r(\mathbf{P},\mathcal{D})$, temperature $\beta$, step size $\eta$, number of particles $K$, total iterations $T$, resampling threshold $K_{thresh}$ (e.g., $K/2$)
\ENSURE A set of weighted particles $\{\mathbf{P}_T^{(k)}, w_T^{(k)}\}_{k=1}^K$ approximating $p(\mathbf{P}|\mathcal{D})$

\STATE Initialize particles $\{\mathbf{P}_0^{(k)}\}_{k=1}^K \sim \mathcal{N}(\mathbf{0}, I)$
\STATE Initialize weights $w_0^{(k)} = \frac{1}{K}$

\FOR{$t = 1$ \TO \ $T$}
    \FOR{$k = 1$ \TO \ $K$}
        \STATE \textbf{1. Langevin Proposal Step:}
        \STATE Sample $\boldsymbol{\epsilon} \sim \mathcal{N}(\mathbf{0},I)$
        \[
        \mathbf{P}_t^{(k)} = \mathbf{P}_{t-1}^{(k)} + \eta \nabla_{\mathbf{P}} r(\mathbf{P}_{t-1}^{(k)},\mathcal{D}) + \sqrt{2\eta}\,\boldsymbol{\epsilon}
        \]
        
        \STATE \textbf{2. Incremental Weight Calculation:}
        \STATE Compute permutation-invariant score:
        \[
        \alpha_t^{(k)} = \frac{1}{|\Pi|} \sum_{\pi \in \Pi} \exp\!\big(\beta r(\pi(\mathbf{P}_t^{(k)}),\mathcal{D})\big)
        \]
        \STATE Update unnormalized weight: $\tilde{w}_t^{(k)} = w_{t-1}^{(k)} \cdot \alpha_t^{(k)}$
    \ENDFOR

    \STATE \textbf{3. Weight Normalization:}
    \[
    W_t = \sum_{j=1}^K \tilde{w}_t^{(j)}, \quad w_t^{(k)} = \frac{\tilde w_t^{(k)}}{W_t}
    \]

    \STATE \textbf{4. Adaptive Resampling (to prevent variance explosion):}
    \STATE Calculate Effective Sample Size: $N_{eff} = \left( \sum_{k=1}^K (w_t^{(k)})^2 \right)^{-1}$
    \IF{$N_{eff} < K_{thresh}$} 
        \STATE Resample $K$ particles from $\{\mathbf{P}_t^{(k)}\}$ with probabilities $\{w_t^{(k)}\}$
        \STATE Reset weights: $w_t^{(k)} = \frac{1}{K}$ for all $k = 1, \dots, K$
    \ENDIF
\ENDFOR

\STATE \textbf{return} $\{\mathbf{P}_T^{(k)}, w_T^{(k)}\}_{k=1}^K$
\end{algorithmic}
\end{algorithm}

\clearpage
\section{Training hyperparameters} \label{append:training_settings}

\subsection{Network Architectures}\label{append:covnet_arch}

\begin{table}[htbp]
\centering
\caption{CoverageRegressor architecture summary (permutation-invariant coverage prediction network).}
\label{tab:covnet-arch}
\renewcommand{\arraystretch}{1.12}
\setlength{\tabcolsep}{4pt}
\begin{tabular}{l l l l}
\toprule
\textbf{Module} & \textbf{Sub-Module} & \textbf{Structure/Computation} & \textbf{Output Dim.} \\
\midrule

\multirow{6}{*}{\shortstack[l]{AP Encoder\\(shared)}}
& Input
& AP coordinates $\mathbf{p}\in\mathbb{R}^{2N}$, reshape to $\mathbf{P}\in\mathbb{R}^{N\times 2}$
& $N\times 2$ \\
& MLP (per-AP)
& $\texttt{Linear}(2\rightarrow32)$ + ReLU
& $32$ \\
&& $\texttt{Linear}(32\rightarrow64)$ + ReLU
& $64$ \\
&& $\texttt{Linear}(64\rightarrow128)$ + ReLU
& $128$ \\
& Pooling
& $\mathbf{f}_{\text{AP}}=\texttt{MeanPool}(\{\phi(\mathbf{p}_k)\}_{k=1}^{K})$
& $128$ \\
\midrule

\multirow{5}{*}{\shortstack[l]{Indoor\\Encoder}}
& Input
& Indoor samples $\mathbf{U}\in\mathbb{R}^{K\times2}$ 
& $K\times2$ \\
& MLP (per-point)
& $\texttt{Linear}(2\rightarrow16)$ + ReLU
& $16$ \\
&& $\texttt{Linear}(16\rightarrow32)$ + ReLU
& $32$ \\
&& $\texttt{Linear}(32\rightarrow64)$
& $64$ \\
& Pooling
& $\mathbf{f}_{\text{in}}=\texttt{MeanPool}(\{\psi(\mathbf{u}_n)\}_{n=1}^{N})$
& $64$ \\
\midrule

\multirow{6}{*}{\shortstack[l]{Coverage\\Predictor}}
& Concat
& $\mathbf{z}=[\mathbf{f}_{\text{AP}},\mathbf{f}_{\text{in}}]$
& $192$ \\
& MLP head
& $\texttt{Linear}(192\rightarrow256)$ + ReLU + Dropout(0.2)
& $256$ \\
&& $\texttt{Linear}(256\rightarrow128)$ + ReLU + Dropout(0.2)
& $128$ \\
&& $\texttt{Linear}(128\rightarrow64)$ + ReLU
& $64$ \\
& Output
& $\texttt{Linear}(64\rightarrow1)$ + Sigmoid
& $1$ \\
& Range
& $\hat{r}\in[0,1]$ (predicted coverage rate)
& $1$ \\
\bottomrule
\end{tabular}
\end{table}

\subsection{Training Hyperparameters}\label{append:covnet_hparams}

\begin{table}[htbp]
\centering
\caption{Training hyperparameters and configuration for CoverageRegressor.}
\label{tab:covnet-hparams}
\renewcommand{\arraystretch}{1.10}
\setlength{\tabcolsep}{10pt}
\begin{tabular}{l c | l c}
\toprule
\textbf{Hyperparameter} & \textbf{Value} & \textbf{Hyperparameter} & \textbf{Value} \\
\midrule
optimizer & Adam & loss & MSE \\
learning rate & $1\times10^{-4}$ & weight decay & $1\times10^{-5}$ \\
epochs & 100 & batch size & 16 \\
validation split & 0.4 & random seed & 42 \\
GPU & 2$\times$NVIDIA GeForce RTX 4090 & activation & ReLU \\
\bottomrule
\end{tabular}
\end{table}

\subsection{Diffusion Optimizer Hyperparameters}
\begin{table}[htbp]
\centering
\caption{Hyperparameters of the diffusion-based AP deployment optimizer.}
\label{tab:diffusion-hparams}
\renewcommand{\arraystretch}{1.10}
\setlength{\tabcolsep}{10pt}
\begin{tabular}{l c}
\toprule
\textbf{Hyperparameter} & \textbf{Value} \\
\midrule
Number of particles & 10 \\
Diffusion steps & 100 \\
Samples per step & 10 \\
Log Linear Noise schedule $(\beta_0, \beta_T)$ & $(1\times10^{-3},\,2)$ \\
Temperature & 1.0 \\
Constraint penalty weight & 10.0 \\
\bottomrule
\end{tabular}
\end{table}

\clearpage
\section{Data generation} \label{append:data_generation}

In wireless network design, accurately characterizing site-specific propagation is essential for
reliable planning and performance evaluation. Ray-tracing has been widely adopted as a
geometry-based and physics-grounded tool to predict location-dependent channel behaviors in
complex indoor environments, where signal attenuation and fluctuations are strongly shaped by
blockage and multipath interactions with walls, doors, windows, and other objects. By explicitly
modeling dominant propagation mechanisms such as line-of-sight (LoS) transmission, reflections,
diffractions, and scattering, ray-tracing enables the computation of high-fidelity radio maps that preserve fine-grained spatial heterogeneity.
This provides a principled basis for tasks such as access point placement, interference-aware
planning, and throughput-driven optimization under realistic indoor layouts.

Following the general ray-tracing formulation \cite{Goldsmith2005Wireless}, the received bandpass signal is modeled as:
\begin{equation}
\begin{aligned}
r_{\mathrm{total}}(t)
= \mathrm{Re} \Bigg\{
\frac{\lambda}{4\pi}
\Bigg[
&\sqrt{G_0}\,u(t)\,\frac{e^{-j\frac{2\pi}{\lambda}d_0}}{d_0}
+\sum_{i=1}^{N_r} R_i \sqrt{G_i}\,u(t-\tau_i)\,\frac{e^{-j\frac{2\pi}{\lambda}d_i}}{d_i} \\
&+\sum_{j=1}^{N_d} L_j(v)\sqrt{G_{d,j}}\,u(t-\tau_j)\,
\frac{e^{-j\frac{2\pi}{\lambda}\left(d_j^{\alpha}+d_j^{\beta}\right)}}{d_j^{\alpha}d_j^{\beta}}
+\sum_{k=1}^{N_s}\sqrt{G_{s,k}\sigma_k}\,u(t-\tau_k)\,
\frac{e^{-j\frac{2\pi}{\lambda}\left(s_k^{\alpha}+s_k^{\beta}\right)}}{\sqrt{4\pi\,s_k^{\alpha}s_k^{\beta}}}
\Bigg]
e^{j2\pi f_c t}
\Bigg\}.
\end{aligned}
\label{eq:general_ray_tracing_superposition}
\end{equation}

Here, $u(t)$ denotes the transmitted complex baseband signal modulating a carrier at frequency $f_c$, with wavelength $\lambda=c/f_c$. We align the time origin to the LoS arrival, so that the LoS component is expressed as $u(t)$. The LoS path has propagation distance $d_0$ and antenna power gain product $G_0$. Since the received signal is formed via complex field superposition, the corresponding amplitude factor appears as $\sqrt{G_0}$. For the $i$-th reflected ray, $d_i$ denotes the total traveled distance and $\tau_i=d_i/c$ is the corresponding delay. The term $G_i$ is the product of the transmit and receive antenna power gains along that ray, and $R_i$ is the reflection coefficient, or the product of reflection coefficients for multi-bounce paths. For diffracted rays, $L_j(v)$ denotes an equivalent diffraction coefficient parameterized by $v$ that captures the diffraction loss and incorporates the geometric spreading associated with the two-leg propagation distances $(d_j^{\alpha},d_j^{\beta})$, and $\tau_j$ is the associated delay. For scattered rays, $\sigma_k$ denotes an effective scattering strength for the $k$-th scatterer, and $(s_k^{\alpha},s_k^{\beta})$ are the distances before and after the scattering point, with $\tau_k$ being the associated delay.

To operationalize the above physics-grounded ray-tracing model at scale, we build a large-scale indoor deployment benchmark with Wireless InSite \cite{RemcomWirelessInSite} over 80 building layouts grouped into four geometric complexity levels (Level 1–4), enabling systematic evaluation across increasing indoor structural complexity. The resulting dataset comprises dense radio maps produced by high-fidelity, physics-based electromagnetic propagation simulations in indoor environments of varying structural complexity, under a unified wireless configuration in terms of carrier frequency and antenna radiation pattern. Specifically, all APs follow the IEEE 802.11g standard \cite{ieee802.11g} operating in the 2.4 GHz band with 20 MHz bandwidth, and we adopt an omnidirectional antenna with 3 dBi gain for every AP during data generation. For each building, we generate multiple deployment instances under different AP counts to cover a range of planning scales, where each instance specifies the 3D AP locations subject to basic feasibility constraints, i.e., within the indoor boundary and valid free-space regions, with a minimum separation between APs. Importantly, the underlying building models in Wireless InSite explicitly encode geometric details and material properties, e.g., walls, doors, and windows, so the generated radio maps reflect realistic attenuation, blockage, and multipath effects driven by the physical environment rather than simplified analytic assumptions.

For each deployment instance, we sample the indoor area on a regular grid with a 0.5 m spatial resolution and run 3D ray-tracing simulations to compute location-wise supervision signals. The simulator outputs complementary radio maps that jointly characterize coverage and network quality, including the pathloss maps, the co-channel interference maps, and the downlink throughput maps. Since these quantities are derived from a consistent, physically grounded propagation and link evaluation pipeline, they provide measurement-grade supervision for learning-based planners, enabling models to capture fine-grained geometry-induced heterogeneity that is typically absent from synthetic or purely pathloss-only datasets. All maps are exported in a consistent and self-descriptive format that includes the grid specification and per-location metric values, enabling reproducible benchmarking and direct use as learning targets for data-driven planners and optimization baselines.

\clearpage
\section{More Experimental Results} \label{append:more_experimental_result}

\subsection{Ablation} \label{append:extra_ablation}

\paragraph{Effect of Temperature.}
Table~\ref{tab:diffusion_temperature_ablation} reports an ablation study on the diffusion sampling temperature $1/\beta$.
According to Claim~\ref{theorem:reward_bottleneck}, the discrepancy between the generated deployment distribution and the target distribution is upper bounded by a dominant term $\mathcal{O}(\beta^4 \varepsilon)$, which originates from the drift mismatch in the reverse-time SDEs induced by reward estimation error.
This theoretical result suggests that smaller $\beta$ is generally preferable, as it suppresses the amplification of reward estimation error and stabilizes the denoising dynamics.

Nevertheless, $\beta$ cannot be arbitrarily small.
When $\beta$ is too small (i.e., the temperature is excessively high), the target distribution
$p(\mathbf{P}, \mathcal{D}) \propto \exp(\beta r(\mathbf{P}, \mathcal{D}))$
becomes overly flat, causing the diffusion sampler to degenerate toward greedy or near-uniform sampling.
As a result, the reward signal becomes insufficiently informative, leading to degraded coverage and solution quality, as observed for small $\beta$ in Table~\ref{tab:diffusion_temperature_ablation}.

Empirically, moderate values of $\beta$ strike a balance between stability and reward sensitivity.
In particular, $\beta \in [0.5, 1.5]$ consistently achieves higher coverage, lower overlap (IOR), improved TQS, and higher success rates, validating our theoretical claim that \emph{reward estimation is the primary bottleneck} in diffusion-based RL and that temperature serves as a crucial control knob for this trade-off.

\begin{table}[ht]
\centering
\caption{Ablation study on temperature for the diffusion sampling. The task involves optimizing the deployment of 2 APs in Building Level 2 with the target coverage threshold 60\%.}
\label{tab:diffusion_temperature_ablation}
\setlength{\tabcolsep}{5pt}
\renewcommand{\arraystretch}{1.08}
\begin{tabular}{c c c c c c}
\toprule
\textbf{Temperature $1/\beta$}
& \textbf{Runtime (s)$\downarrow$}
& \textbf{Coverage (\%)$\uparrow$}
& \textbf{IOR$\downarrow$}
& \textbf{TQS$\uparrow$}
& \textbf{Success (\%)$\uparrow$} \\
\midrule
0.1 & 130.25 $\pm$ 71.56 & 55.37 $\pm$ 7.60 & 0.42 $\pm$ 0.15 & 26.54 $\pm$ 4.22 & 87.0 \\
0.2 & 52.04  $\pm$ 22.26 & 57.14 $\pm$ 6.40 & 0.31 $\pm$ 0.08 & 31.87 $\pm$ 3.15 & 86.5 \\
0.5 & 37.28  $\pm$ 28.07 & 61.78 $\pm$ 4.36 & 0.23 $\pm$ 0.10 & 38.42 $\pm$ 2.96 & 88.5 \\
1.5 & 16.47  $\pm$ 10.98 & 63.22 $\pm$ 7.35 & 0.20 $\pm$ 0.06 & 42.18 $\pm$ 3.84 & 91.0 \\
2.0 & 14.31  $\pm$ 7.47  & 63.37 $\pm$ 8.96 & 0.22 $\pm$ 0.09 & 41.25 $\pm$ 4.12 & 94.0 \\
5.0 & 17.22  $\pm$ 7.09  & 62.91 $\pm$ 4.45 & 0.21 $\pm$ 0.11 & 40.64 $\pm$ 5.37 & 89.0 \\
\bottomrule
\end{tabular}
\end{table}

\subsection{Simulation Results for Level 1 to 4}\label{append:Simulation Results for different levels}
\begin{table}[H]
\centering
\caption{Performance comparison under different AP deployment scales in Building Level 1. The simulation results are reported as (mean$\pm$std) with color coding indicating performance rankings: the \textit{top three} methods in each metric with different number of APs are highlighted in \colorbox{green!25}{(1st)}, \colorbox{blue!15}{(2nd)}, and \colorbox{gray!20}{(3rd)}. The \worst{worst} performance in each column is underlined.}
\label{tab:b18_apscale}
\setlength{\tabcolsep}{5pt}
\begin{tabular}{ccccccc}
\toprule
\textbf{APs} & \textbf{Method} 
& \textbf{Runtime (s)$\downarrow$} 
& \textbf{Coverage (\%)$\uparrow$} 
& \textbf{IOR$\downarrow$} 
& \textbf{TQS$\uparrow$} 
& \textbf{Success (\%)$\uparrow$} \\
\midrule

\multirow{10}{*}{1}
& Diffusion       & \first{42.47±25.64} & \first{87.20±4.67} & -- & \first{54.08±0.20} & \first{94.0} \\
& Weighted sampling (Gaussian)             & 279.42±206.35 & 79.28±0.77 & -- & \third{53.97±0.17} & \third{92.0} \\
& Weight sampling (Langevin)         & 152.71±122.25 & 79.29±0.99 & -- & \second{53.99±0.11} & \second{93.0} \\
& DeepSeek-V3.2    & 240.75±26.47 & 82.07±1.38 & -- & 48.88±0.50 & 91.0 \\
& GPT-4-turbo      & 564.39±170.81 & 71.07±0.46 & -- & 52.34±3.18 & 88.0 \\
& Claude-3-Haiku   & \third{111.34±2.52} & 80.53±4.02 & -- & 49.01±0.45 & 90.0 \\
& Gemini-2.5-flash & \second{102.35±54.66} & \second{84.12±0.62} & -- & 48.57±0.31 & \third{92.0} \\
& Grok-4           & \worst{1842.33±1269.39} & \third{83.60±5.21} & -- & 49.13±0.62 & 91.0 \\
& Qwen-plus        & 121.72±13.46 & 83.50±4.56 & -- & 49.08±0.25 & 90.0 \\
& Convex Optimization      & 1355.89±970.75 & \worst{48.55±14.99} & -- & \worst{45.27±2.37} & \worst{75.0} \\
\midrule

\multirow{10}{*}{2}
& Diffusion       & 11.99±6.41 & \first{93.47±2.57} & \first{0.92±0.16} & \first{37.68±6.26} & \first{92.0} \\
& Weighted sampling (Gaussian)             & 28.57±8.66 & 92.54±2.39 & \third{0.97±0.22} & \second{37.57±1.23} & \third{91.0} \\
& Weight sampling (Langevin)         & 19.57±8.66 & 92.84±1.76 & 1.19±0.36 & 35.30±5.45 & \second{91.5} \\
& DeepSeek-V3.2    & \second{8.89±1.11} & 90.68±0.33 & 1.07±0.47 & 33.77±5.11 & 89.0 \\
& GPT-4-turbo      & 34.83±55.74 & 90.71±1.98 & 1.51±0.42 & 31.44±1.20 & 88.5 \\
& Claude-3-Haiku   & \first{7.06±0.37} & 91.41±0.81 & \second{0.93±0.46} & \third{36.29±5.40} & 90.0 \\
& Gemini-2.5-flash & \third{9.53±0.66} & \third{92.89±0.18} & 1.38±0.38 & 32.44±4.20 & 90.5 \\
& Grok-4           & 19.29±3.14 & \second{92.99±1.93} & 1.39±0.24 & 29.35±5.18 & \third{91.0} \\
& Qwen-plus        & \worst{48.48±0.79} & 92.12±0.51 & 1.50±0.02 & 27.04±0.00 & 89.5 \\
& Convex Optimization      & 34.52±6.96 & \worst{77.56±15.74} & \worst{3.36±2.70} & \worst{17.02±15.46} & \worst{70.0} \\
\midrule

\multirow{10}{*}{3}
& Diffusion       & \first{41.49±11.10} & 99.28±1.36 & \first{5.21±1.71} & \first{18.21±3.64} & \first{90.0} \\
& Weighted sampling (Gaussian)             & 124.00±93.86 & 98.01±1.23 & \second{6.04±1.81} & \second{16.04±9.32} & 88.5 \\
& Weight sampling (Langevin)         & \third{100.46±2.27} & 96.32±2.04 & 8.36±2.89 & \third{9.16±8.89} & \second{89.5} \\
& DeepSeek-V3.2    & 112.56±0.89 & \second{99.81±0.18} & 9.48±0.48 & 3.50±0.95 & 87.5 \\
& GPT-4-turbo      & \worst{359.42±134.22} & 98.19±0.68 & 8.39±2.29 & 8.25±9.19 & 86.0 \\
& Claude-3-Haiku   & 210.78±1.82 & \third{99.79±0.29} & \third{6.53±2.88} & 5.18±1.01 & 88.0 \\
& Gemini-2.5-flash & 114.25±1.24 & \first{99.86±0.10} & 9.22±0.28 & 4.06±0.61 & 88.5 \\
& Grok-4           & 123.30±3.10 & 99.66±0.37 & 9.15±0.63 & 4.11±0.72 & \third{89.0} \\
& Qwen-plus        & 212.07±1.03 & 99.60±0.24 & 9.32±0.07 & 3.71±0.00 & 87.0 \\
& Convex Optimization      & \second{47.74±8.92} & \worst{86.35±8.70} & \worst{10.36±2.89} & \worst{1.95±1.88} & \worst{72.0} \\
\bottomrule
\end{tabular}
\end{table}

\begin{table}[H]
\centering
\caption{Performance comparison under different AP deployment scales in Building Level 2. The simulation results are reported as (mean$\pm$std) with color coding indicating performance rankings: the \textit{top three} methods in each metric with different number of APs are highlighted in \colorbox{green!25}{(1st)}, \colorbox{blue!15}{(2nd)}, and \colorbox{gray!20}{(3rd)}. The \worst{worst} performance in each column is underlined.}
\label{tab:b2_apscale}
\setlength{\tabcolsep}{5pt}
\begin{tabular}{ccccccc}
\toprule
\textbf{APs} & \textbf{Method} 
& \textbf{Runtime (s)$\downarrow$} 
& \textbf{Coverage (\%)$\uparrow$} 
& \textbf{IOR$\downarrow$} 
& \textbf{TQS$\uparrow$} 
& \textbf{Success (\%)$\uparrow$} \\
\midrule

\multirow{10}{*}{1}
& Diffusion       & \first{4.13±2.29} & \first{26.91±0.56} & -- & \third{43.89±22.85} & \first{91.0} \\
& Weighted sampling (Gaussian)             & 91.58±66.91 & 24.14±1.12 & -- & 29.38±26.82 & 86.0 \\
& Weight sampling (Langevin)        & 54.11±46.05 & 24.81±0.44 & -- & 36.61±21.37 & 88.0 \\
& DeepSeek-V3.2    & 7.69±1.35 & \worst{21.08±4.31} & -- & 43.07±11.76 & 85.0 \\
& GPT-4-turbo      & 167.09±164.44 & 23.04±2.24 & -- & 19.20±26.30 & 87.0 \\
& Claude-3-Haiku   & \second{5.08±0.10} & 24.08±1.20 & -- & 43.54±20.34 & 86.7 \\
& Gemini-2.5-flash & \third{6.05±0.28} & \second{26.08±0.11} & -- & 43.07±22.55 & 89.0 \\
& Grok-4           & 16.33±3.40 & 25.08±0.00 & -- & \second{46.55±20.45} & \second{90.0} \\
& Qwen-plus        & 7.01±0.49 & \third{25.81±0.60} & -- & \first{47.49±18.37} & \third{89.3} \\
& Convex Optimization      & \worst{1186.33±1606.52} & 24.87±0.64 & -- & \worst{16.17±22.95} & \worst{70.0} \\
\midrule

\multirow{10}{*}{2}
& Diffusion       & \second{49.87±16.35} & \second{53.12±0.44} & 0.50±0.29 & 34.02±4.97 & \first{89.0} \\
& Weighted sampling (Gaussian)             & 66.40±5.82 & 50.82±0.95 & 0.40±0.36 & 25.86±15.08 & \third{87.0} \\
& Weight sampling (Langevin)         & \third{51.40±1.22} & \third{52.92±0.47} & 0.83±0.44 & 31.22±13.53 & \second{88.0} \\
& DeepSeek-V3.2    & \first{14.88±3.73} & 50.81±1.44 & \second{0.12±0.10} & \third{36.91±3.13} & 84.0 \\
& GPT-4-turbo      & \worst{483.96±569.41} & 51.40±1.22 & 0.33±0.27 & \second{37.94±11.85} & 86.0 \\
& Claude-3-Haiku   & 212.04±0.41 & \first{54.03±12.24} & \first{0.06±0.05} & \first{40.89±4.48} & 83.0 \\
& Gemini-2.5-flash & 222.32±0.40 & 48.35±12.52 & 0.29±0.15 & 33.75±5.45 & 82.0 \\
& Grok-4           & 219.43±1.01 & 50.31±2.72 & \third{0.15±0.03} & 37.82±1.81 & 84.7 \\
& Qwen-plus        & 332.31±1.35 & 49.05±11.88 & 0.29±0.04 & 37.49±1.90 & 83.3 \\
& Convex Optimization      & 140.98±8.95 & \worst{32.80±6.91} & \worst{1.71±0.11} & \worst{8.28±9.85} & \worst{55.0} \\
\midrule

\multirow{10}{*}{3}
& Diffusion       & \first{56.37±6.12} & \first{89.89±1.16} & \third{1.46±0.74} & \third{19.84±10.61} & \first{88.0} \\
& Weighted sampling (Gaussian)             & \third{88.08±49.87} & 86.53±1.10 & \second{1.17±0.33} & 18.50±3.51 & 84.0 \\
& Weight sampling (Langevin)         & \second{72.08±9.87} & \third{86.94±1.04} & 1.63±0.14 & \second{19.91±5.83} & 86.0 \\
& DeepSeek-V3.2    & 116.87±0.49 & 85.96±6.27 & 1.63±1.02 & 14.21±9.21 & 80.0 \\
& GPT-4-turbo      & \worst{395.99±172.93} & 79.81±1.70 & 1.59±0.06 & \first{26.96±11.74} & 78.0 \\
& Claude-3-Haiku   & 215.10±0.26 & 81.53±7.16 & \first{0.93±0.21} & 15.74±6.06 & 76.0 \\
& Gemini-2.5-flash & 119.65±2.14 & 86.18±5.49 & 2.46±0.68 & 8.55±3.26 & \second{87.0} \\
& Grok-4           & 215.84±2.89 & 82.80±2.00 & 2.45±0.41 & 9.13±3.81 & \third{86.7} \\
& Qwen-plus        & 216.52±0.61 & \second{87.92±9.62} & 2.02±0.33 & 6.67±2.90 & 85.3 \\
& Convex Optimization      & 304.28±15.95 & \worst{49.03±8.00} & \worst{2.75±1.02} & \worst{5.00±1.25} & \worst{50.0} \\
\bottomrule
\end{tabular}
\end{table}

\begin{table}[htbp]
\centering
\caption{Performance comparison under different AP deployment scales in Building Level 3. The simulation results are reported as (mean$\pm$std) with color coding indicating performance rankings: the \textit{top three} methods in each metric with different number of APs are highlighted in \colorbox{green!25}{(1st)}, \colorbox{blue!15}{(2nd)}, and \colorbox{gray!20}{(3rd)}. The \worst{worst} performance in each column is underlined.}
\label{tab:b5_apscale}
\setlength{\tabcolsep}{5pt}
\begin{tabular}{ccccccc}
\toprule
\textbf{APs} & \textbf{Method} 
& \textbf{Runtime (s)$\downarrow$} 
& \textbf{Coverage (\%)$\uparrow$} 
& \textbf{IOR$\downarrow$} 
& \textbf{TQS$\uparrow$} 
& \textbf{Success (\%)$\uparrow$} \\
\midrule

\multirow{10}{*}{1}
& Diffusion       & \first{32.52±1.41}  & \second{37.77±0.59} & -- & \first{37.38±9.56} & \first{88.0} \\
& Weighted sampling (Gaussian)             & 143.46±13.38 & 32.54±0.75 & -- & 34.14±13.43 & 85.7 \\
& Weight sampling (Langevin)         & \third{116.28±77.23} & 32.86±0.45 & -- & \second{35.98±14.21} & \third{86.7} \\
& DeepSeek-V3.2    & 284.18±172.58 & \third{35.97±5.45} & -- & 26.14±17.71 & 84.0 \\
& GPT-4-turbo      & 165.74±150.23 & 32.56±0.80 & -- & \third{35.88±12.47} & 85.3 \\
& Claude-3-Haiku   & 182.57±43.97 & 34.04±3.61 & -- & 21.63±16.19 & 83.3 \\
& Gemini-2.5-flash & \second{51.00±38.65} & 35.73±2.95 & -- & 24.14±17.05 & 84.7 \\
& Grok-4           & 157.45±71.34 & \first{39.17±8.07} & -- & 33.67±8.58 & \second{87.3} \\
& Qwen-plus        & 473.72±228.66 & 32.08±5.40 & -- & 33.06±10.37 & 81.7 \\
& Convex Optimization      & \worst{2963.72±1623.70} & \worst{30.08±1.37} & -- & \worst{20.27±11.56} & \worst{55.0} \\
\midrule

\multirow{10}{*}{2}
& Diffusion       & \first{37.04±34.83} & \first{66.38±1.97} & \first{0.04±0.03} & \first{42.28±4.90} & \first{84.0} \\
& Weighted sampling (Gaussian)             & \third{46.51±45.86} & 62.35±1.74 & \third{0.17±0.08} & 41.09±3.51 & \third{82.0} \\
& Weight sampling (Langevin)         & \second{38.39±60.60} & \third{63.40±1.13} & 0.28±0.14 & 41.44±0.88 & \second{83.0} \\
& DeepSeek-V3.2    & 84.51±74.85 & 62.68±4.42 & 0.18±0.15 & \third{41.81±0.42} & 80.7 \\
& GPT-4-turbo      & \worst{820.09±763.82} & 62.36±1.23 & 0.19±0.08 & \second{42.02±1.22} & 81.3 \\
& Claude-3-Haiku   & 125.56±14.32 & 60.68±4.42 & 0.32±0.24 & 36.38±6.47 & 76.7 \\
& Gemini-2.5-flash & 218.80±8.29 & 61.46±1.01 & 0.62±0.41 & 41.40±0.33 & 79.3 \\
& Grok-4           & 134.18±93.97 & 61.60±10.63 & \worst{0.70±0.64} & 25.68±13.62 & 78.0 \\
& Qwen-plus        & 225.91±10.36 & \second{65.86±3.04} & \second{0.06±0.01} & 37.85±0.14 & 80.0 \\
& Convex Optimization      & 138.18±27.92 & \worst{36.17±9.91} & 0.56±0.11 & \worst{15.17±21.96} & \worst{50.0} \\
\midrule

\multirow{10}{*}{3}
& Diffusion       & \first{56.79±25.05} & \first{77.24±3.16} & \first{0.42±0.19} & \first{37.89±2.60} & \first{80.7} \\
& Weighted sampling (Gaussian)             & \third{76.77±25.93} & 70.57±2.76 & 1.07±0.32 & 32.03±8.56 & \third{76.7} \\
& Weight sampling (Langevin)         & \second{75.97±6.01}  & 71.61±2.98 & \second{0.87±0.27} & \second{34.69±2.30} & \second{78.7} \\
& DeepSeek-V3.2    & 296.97±283.76 & \second{76.43±3.32} & 1.66±1.48 & 27.03±16.26 & 75.3 \\
& GPT-4-turbo      & 540.75±133.80 & 70.07±1.63 & \third{1.02±0.40} & \third{34.38±3.96} & 74.0 \\
& Claude-3-Haiku   & 242.51±31.86 & 75.71±5.82 & 2.09±1.02 & 26.66±8.15 & 72.7 \\
& Gemini-2.5-flash & 102.56±40.70 & 74.25±1.53 & 1.63±0.95 & 24.43±10.48 & 73.3 \\
& Grok-4           & 694.91±653.31 & \third{76.36±3.02} & 2.03±0.73 & 20.64±6.75 & 74.7 \\
& Qwen-plus        & 322.09±219.62 & 75.03±4.99 & \worst{2.71±0.68} & \worst{10.60±10.04} & 73.7 \\
& Convex Optimization      & \worst{899.26±41.49} & \worst{49.20±12.35} & 1.30±1.13 & 19.35±4.22 & \worst{48.0} \\
\bottomrule
\end{tabular}
\end{table}

\clearpage
\subsection{Prompt Design for Agentic Reasoning} \label{append:demo_of_agentc_reasoning}

\tcbset{
  promptbox/.style={
    colback=gray!8,      
    colframe=gray!60,    
    boxrule=0.6pt,
    arc=2pt,
    left=6pt,right=6pt,top=6pt,bottom=6pt,
    fontupper=\ttfamily\small,
    breakable
  }
}

\begin{tcolorbox}[promptbox,title=\textbf{Initialization Prompt}]
Role \& Mission: You are a wireless network planning assistant. Given the indoor layout and a fixed number of APs, propose AP locations to maximize coverage probability under a received power threshold.
\begin{lstlisting}
Building Analysis:
Indoor area coordinates: {indoor_coordinates}
Total indoor area: {indoor_area_size} grid cells

Floor Information: {floor_info}
where the floor information is a 2d array with 0 to 3:
- 0 = free space (preferred for AP placement)
- 1 = wall (avoid placing AP here)
- 2 = window (avoid placing AP here)
- 3 = door (avoid placing AP here)
\end{lstlisting}

Before generating coordinates, you must think step-by-step.
\begin{lstlisting}
1) Deployment feasibility constraints:
- Keep the number of APs fixed:{N}.
- Each AP must be placed in indoor free-space cells only (value 0).
- Minimum separation between the APs: {d_min}.
- AP coordinates must remain within the building boundary.

2) Coverage definition (pathloss-based):
- Pathloss should be lower than {PL_threshold}.
- Interference should be lower than {I_threshold} .
- Throughput should be larger than {T_threshold}.
\end{lstlisting}

\vspace{4pt}

You must keep the following wireless propagation principles:

\vspace{4pt}
1) Distance-dependent attenuation:
   Signal strength decays with distance. Avoid leaving large uncovered areas far from any AP.

2) Wall and obstacle penetration loss:
   Walls/obstacles introduce additional attenuation. Placing an AP behind multiple walls rarely improves coverage in the target region.

3) Line-of-sight preference:
   Prioritize AP locations that provide clear LoS paths to large open regions or key corridors.

4) NLoS compensation strategy:
   If a region is blocked, cover it by placing an AP inside the region or near entrances/openings, rather than relying on deep penetration.

5) Coverage overlap control:
   Some overlap is needed for smooth coverage, but excessive clustering wastes APs. Spread APs to maximize marginal coverage gain.

6) Spatial diversity:
   Prefer placing APs in geometrically diverse positions (different rooms/corridors) to reduce redundancy and shadowed zones.

7) Boundary awareness:
   Avoid placing APs too close to walls/corners unless needed to cover edge regions; central elevated positions often produce broader coverage.

8) Iterative improvement rule:
   Use the verifier feedback to identify the largest uncovered region first, then relocate the most redundant AP toward that region.

   \vspace{4pt}
Output requirement:

- Always output exactly N AP locations in the required structured format in '[[X1, Y1, Z1], [X2, Y2, Z2], ...]' format without any other text.
\end{tcolorbox}

\begin{tcolorbox}[promptbox,title=\textbf{Refinement with Verifier Feedback}]
Role \& Mission: You receive the current best deployment and its evaluation results from the physical verifier.
If such feedback is available, you MUST treat the current best deployment as the baseline and perform a refinement rather than re-planning from scratch.
\begin{lstlisting}
Current best deployment: {best_AP_deployment_json}

Verifier feedback for the baseline:
- Coverage score: {best_coverage_score}
- Constraint status: {constraint_summary}
- Visual summary from coverage maps: {map_summary}
\end{lstlisting}

Refinement policy:

- Identify the largest uncovered region (4: not covered, 5: covered) indicated by the feedback.

- Prefer relocating APs that contribute the least marginal gain or lie in redundant clusters.

- Move APs toward the uncovered region, while preserving already well-covered areas.

- Avoid moving many APs simultaneously. Update only a small subset of APs in one iteration.

- If interference violations are reported, reduce excessive overlap by increasing spatial separation among nearby APs.

- If throughput violations are reported, prioritize improving service to those regions by placing an AP closer with fewer obstructions.

\vspace{4pt}

Output requirement:

- Always output exactly N AP locations in the required structured format in '[[X1, Y1, Z1], [X2, Y2, Z2], ...]' format without any other text.

- Do not output intermediate analysis. Provide only the final updated deployment.

\end{tcolorbox}

\clearpage
\subsection{Coverage Physics Consistency}
\label{subsec:consistency}

\begin{figure}[ht]
    \centering
    \includegraphics[width=0.9\linewidth]{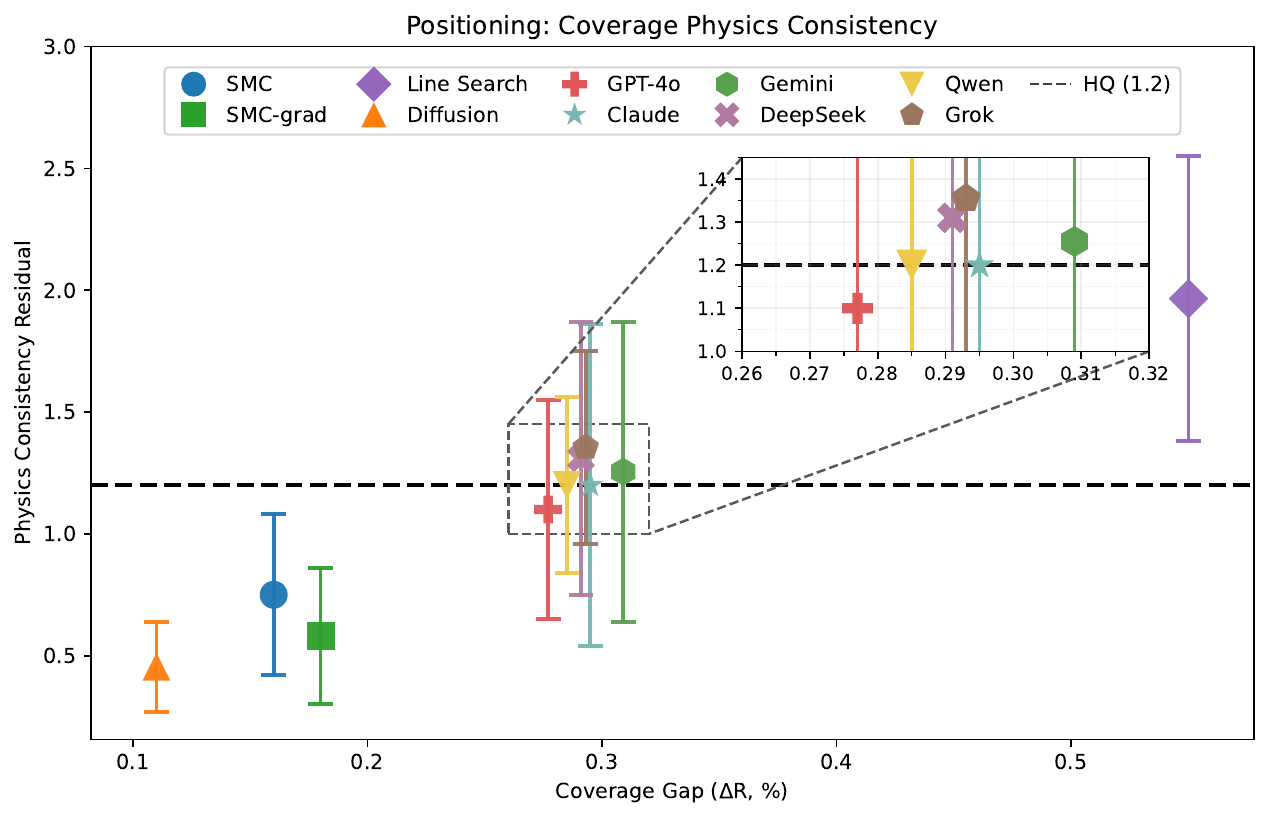}
    \caption{Coverage Physics Consistency.}
    \label{fig:Consistency}
\end{figure}

\textbf{(a) Interference Consistency Residual.}
We measure whether $\mathbf{P}$ satisfies the interference constraint $I(\mathbf{r}) \le \xi$ by the normalized violation residual
\begin{equation}
e_{I}(\mathbf{P})
=\frac{1}{|\Omega|}\sum_{\mathbf{r}\in\Omega}
\left[\max\!\left(0,\frac{I(\mathbf{r})-\xi}{\xi}\right)\right]^{2}.
\label{eq:eI}
\end{equation}

\textbf{(b) Throughput Consistency Residual.}
We quantify the feasibility of meeting the minimum throughput constraint $T(\mathbf{r}) \ge T_{\min}$ by
\begin{equation}
e_{T}(\mathbf{P})
=\frac{1}{|\Omega|}\sum_{\mathbf{r}\in\Omega}
\left[\max\!\left(0,\frac{T_{\min}-T(\mathbf{r})}{T_{\min}}\right)\right]^{2}.
\label{eq:eT}
\end{equation}

\textbf{(c) Minimum-Separation Consistency Residual.}
To evaluate whether the minimum separation constraint $\|\mathbf{p}_{i}-\mathbf{p}_{j}\|\ge d_{\min}$ is satisfied, we compute the averaged pairwise violation residual
\begin{equation}
e_{d}(\mathbf{P})
=\frac{2}{N(N-1)}\sum_{i<j}
\left[\max\!\left(0,\frac{d_{\min}-\|\mathbf{p}_{i}-\mathbf{p}_{j}\|}{d_{\min}}\right)\right]^{2}.
\label{eq:ed}
\end{equation}

\textbf{(d) Boundary Consistency Residual.}
To ensure all APs lie within the admissible indoor region $\mathcal{D}$, we measure the average out-of-domain distance using the projection operator $\Pi_{\mathcal{D}}(\cdot)$ as
\begin{equation}
e_{b}(\mathbf{P})
=\frac{1}{N}\sum_{n=1}^{N}
\left\|\mathbf{p}_{n}-\Pi_{\mathcal{D}}(\mathbf{p}_{n})\right\|_{2}^{2}.
\label{eq:eb}
\end{equation}

\textbf{Physics-level consistency.}
We aggregate the above constraint residuals into a unified physics-level consistency score:
\begin{equation}
e_{\mathrm{phy}}(\mathbf{P})
=\lambda_{I}e_{I}(\mathbf{P})
+\lambda_{T}e_{T}(\mathbf{P})
+\lambda_{d}e_{d}(\mathbf{P})
+\lambda_{b}e_{b}(\mathbf{P}),
\label{eq:ephy}
\end{equation}
where $\lambda_{I},\lambda_{T},\lambda_{d},\lambda_{b}$ denote weighting coefficients balancing heterogeneous constraint types.

\clearpage
\subsection{Visualization of Intermediate Reasoning and Iterative Optimization States}\label{append:Visualization of Intermediate Reasoning and Iterative Optimization States}

\begin{figure}[ht]
    \centering
    \includegraphics[width=0.8\linewidth]{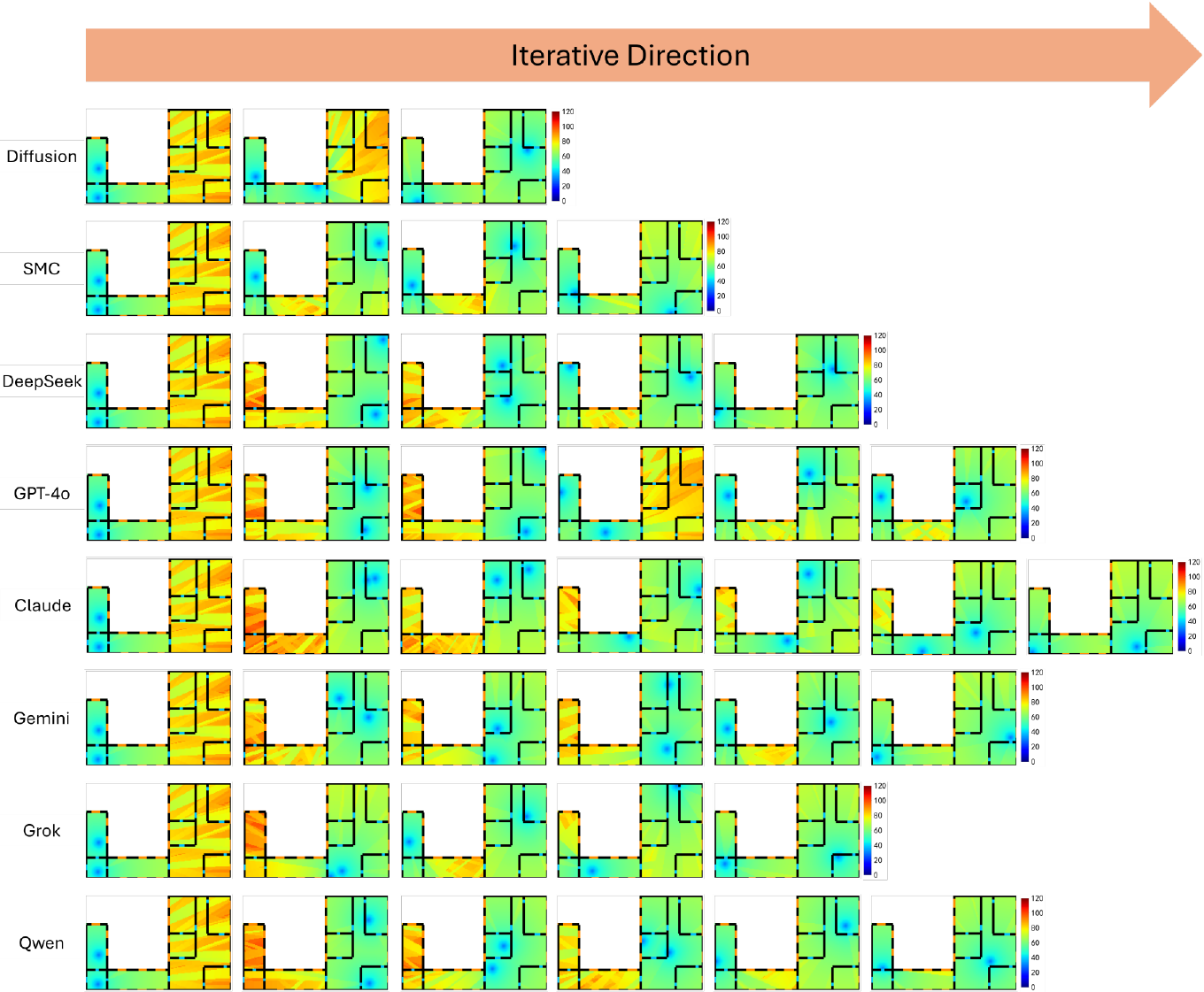}
    \caption{Visualization of the iterative reasoning and optimization trajectories in the Level 1 Building. The task requires coordinating 2 APs to achieve 100\% coverage threshold. Starting from a unified stochastic initialization, the subplots compare the state transitions of the diffusion sampling, weighted sampling, and various LLMs. The sequences illustrate how each baseline navigates the complex structural indoor layout to progressively refine AP coordinates.}
    \label{fig:reasoning_level_1}
\end{figure}

\begin{figure}[ht]
    \centering
    \includegraphics[width=0.8\linewidth]{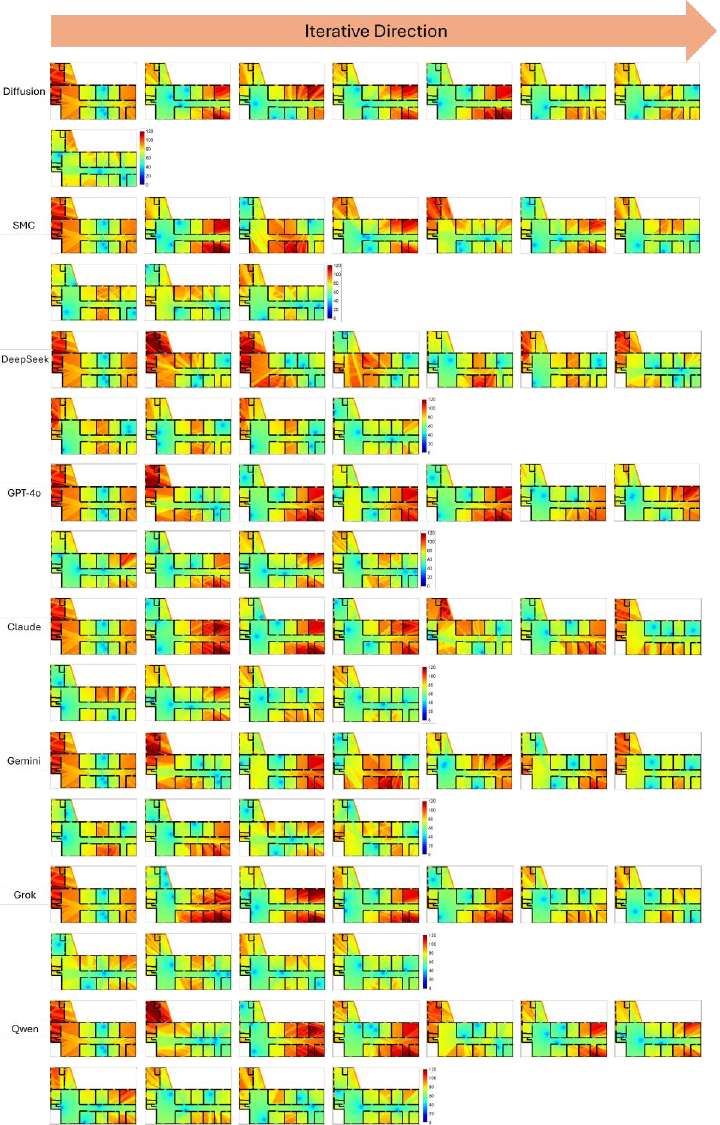}
    \caption{Visualization of the iterative reasoning and optimization trajectories in the Level 2 Building. The task requires coordinating 5 APs to achieve 99\% coverage threshold. Starting from a unified stochastic initialization, the subplots compare the state transitions of the diffusion sampling, weighted sampling, and various LLMs. The sequences illustrate how each baseline navigates the complex structural indoor layout to progressively refine AP coordinates.}
    \label{fig:reasoning_level_2}
\end{figure}

\begin{figure}[ht]
    \centering
    \includegraphics[width=0.8\linewidth]{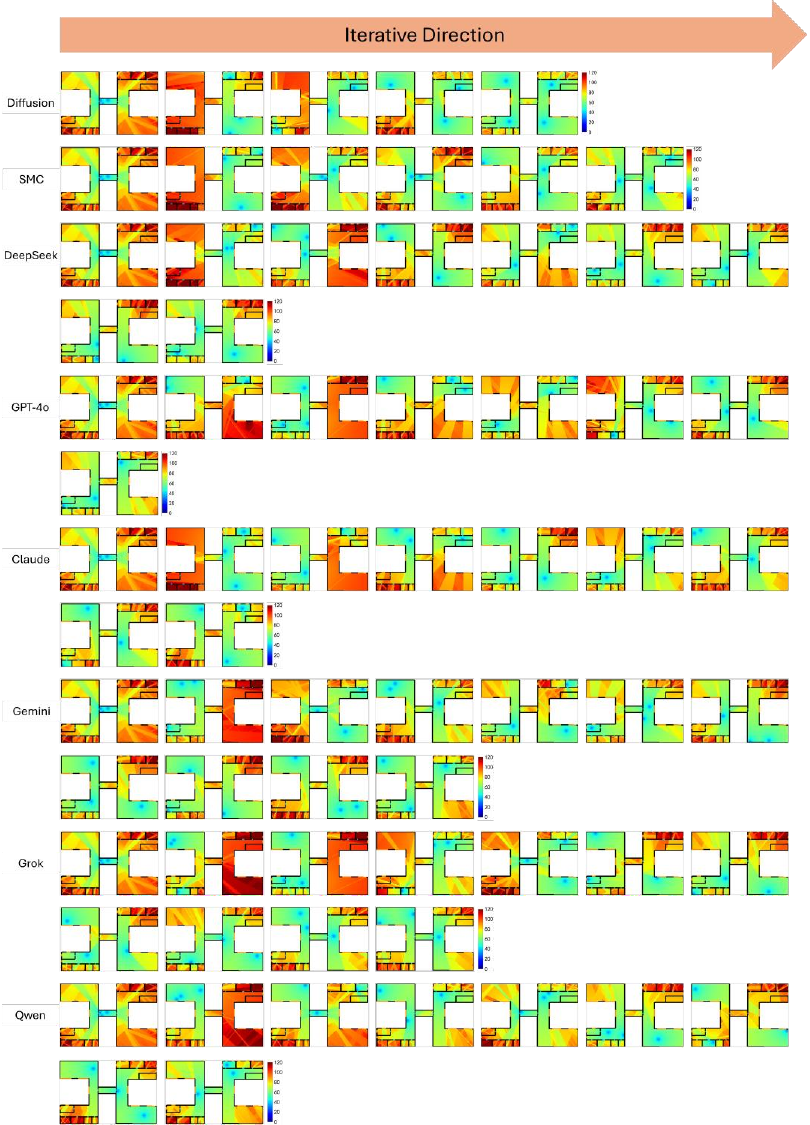}
    \caption{Visualization of the iterative reasoning and optimization trajectories in the Level 3 Building. The task requires coordinating 6 APs to achieve 95\% coverage threshold. Starting from a unified stochastic initialization, the subplots compare the state transitions of the diffusion sampling, weighted sampling, and various LLMs. The sequences illustrate how each baseline navigates the complex structural indoor layout to progressively refine AP coordinates.}
    \label{fig:reasoning_level_3}
\end{figure}

\begin{figure}[ht]
    \centering
    \includegraphics[width=0.8\linewidth]{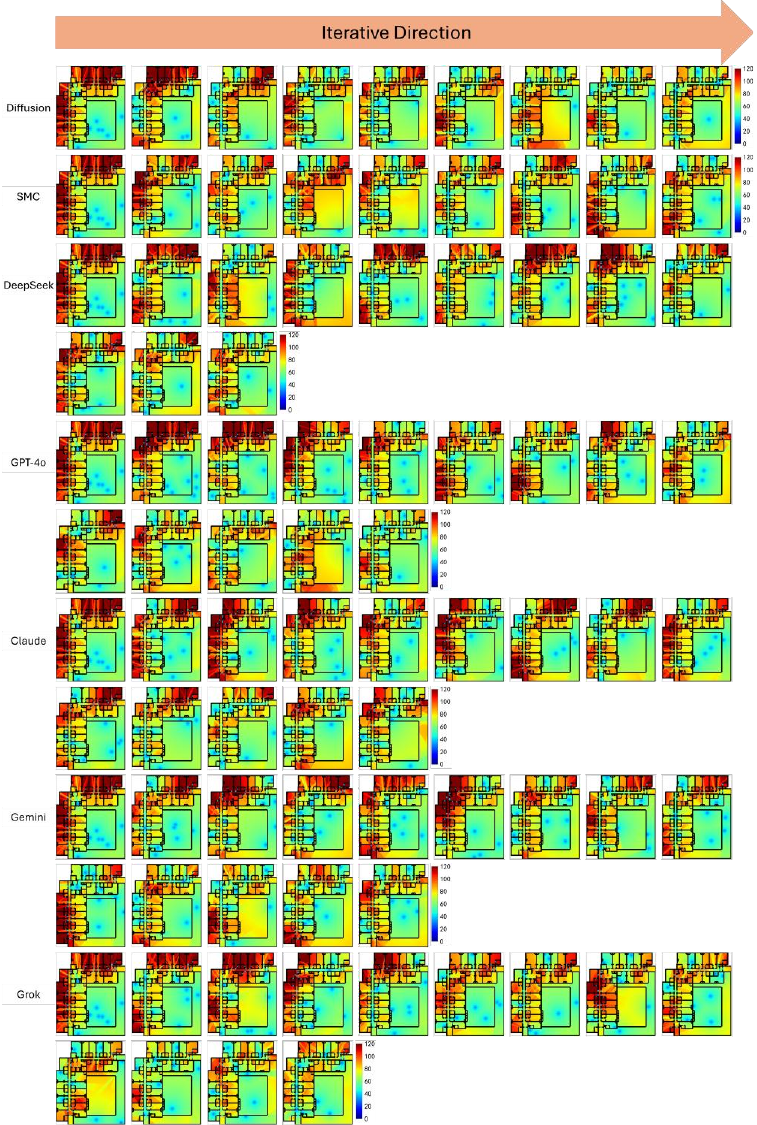}
    \caption{Visualization of the iterative reasoning and optimization trajectories in the Level 4 Building. The task requires coordinating 11 APs to achieve 90\% coverage threshold. Starting from a unified stochastic initialization, the subplots compare the state transitions of the diffusion sampling, weighted sampling, and various LLMs. The sequences illustrate how each baseline navigates the complex structural indoor layout to progressively refine AP coordinates.}
    \label{fig:reasoning_level_4}
\end{figure}


\end{document}